\documentclass{article}

\usepackage{todonotes}
\usepackage{amsmath}
\usepackage{amsthm}
\usepackage{amssymb}
\usepackage{amsfonts} 
\usepackage{mathrsfs}
\usepackage{graphicx}
\usepackage{epstopdf}
\usepackage{nicefrac}

\usepackage{enumitem}

\newtheorem{lemma}{Lemma}

\newtheorem{proposition}{Proposition}
\newtheorem{theorem}{Theorem}

\newtheoremstyle{named}{}{}{\itshape}{}{\bfseries}{.}{.5em}{#1 #3}
\theoremstyle{named}
\newtheorem*{namthm*}{Theorem}

\usepackage{mathtools}

\DeclareMathOperator{\sign}{sign}
\usepackage{bm}
\usepackage{color,soul}

\newcommand{\dc}{\mathcal{DC}}
\newcommand{\pldc}{\textrm{pl-}\mathcal{DC}}

\usepackage{microtype}
\usepackage{graphicx}
\usepackage{subfigure}
\usepackage{booktabs} 

\usepackage{hyperref}

\usepackage[accepted]{icml2020}

\icmltitlerunning{Piecewise Linear Regression via a Difference of Convex Functions}

\usepackage{pifont}
\newcommand{\cmark}{\ding{51}}%
\newcommand{\xmark}{\ding{55}}%

\begin{document}

\twocolumn[
\icmltitle{Piecewise Linear Regression via a Difference of Convex Functions}



\icmlsetsymbol{equal}{*}

\begin{icmlauthorlist}
\icmlauthor{Ali Siahkamari \textsuperscript{*}}{to}
\icmlauthor{Aditya Gangrade \textsuperscript{*}}{two}
\icmlauthor{Brian Kulis}{to}
\icmlauthor{Venkatesh Saligrama}{to}
\end{icmlauthorlist}

\icmlaffiliation{to}{Department of Electrical and Computer Engineering, Boston University}
\icmlaffiliation{two}{Division of Systems Engineering, Boston University}

\icmlcorrespondingauthor{Ali Siahkamari}{siaa@bu.edu}

\vskip 0.3in
]


\printAffiliationsAndNotice{\icmlEqualContribution} 

\begin{abstract}
     
     We present a new piecewise linear regression methodology that utilizes fitting a \emph{difference of convex} functions (DC functions) to the data. These are functions $f$ that may be represented as the difference $\phi_1 - \phi_2$ for a choice of \emph{convex} functions $\phi_1, \phi_2$. The method proceeds by estimating piecewise-liner convex functions, in a manner similar to max-affine regression, whose difference approximates the data. The choice of the function is regularised by a new seminorm over the class of DC functions that controls the $\ell_\infty$ Lipschitz constant of the estimate. The resulting methodology can be efficiently implemented via Quadratic programming \emph{even in high dimensions}, and is shown to have close to minimax statistical risk. We empirically validate the method, showing it to be practically implementable, and to have comparable performance to existing regression/classification methods on real-world datasets. 
\end{abstract}

\section{Introduction}\label{sec:intro}


The multivariate nonparametric regression problem is a fundamental statistical primitive, with a vast history and many approaches. We adopt the following setup: given a dataset, $\{(\bm x_i, y_i)\}_{i \in [n]}$, where $\bm x_i \in \mathbb{R}^d$ are predictors, assumed drawn i.i.d.~from a law $P_X,$ and $y_i$ are responses such that $y_i = f(\bm x_i) + \varepsilon_i,$ for a bounded, centered, independent random noise $\varepsilon_i$, and bounded $f$, the goal is to recover an estimate $\hat{f}$ of $f$ such that on new data, the squared error $\mathbb{E}[(y- \hat{f}(\bm x))^2]$ is small. 

The statistical challenge of the problem lies in the fact that $f$ is only weakly constrained - for instance, $f$ may only be known to be differentiable, or Lipschitz. In addition, the problem is algorithmically challenging in high-dimensions, and many approaches to the univariate problem do not scale well with the dimension $d$. For instance, piecewise linear regression methods typically involve a prespecified grid, and thus the number of grid points, or knots, grows exponentially with dimension. Similarly, methods like splines typically require both stronger smoothness guarantees and exponentially more parameters to fit with dimension in order to avoid singularities in the estimate.

This paper is concerned with regression over the class of functions that are \emph{differences of convex} functions, i.e., DC functions. These are functions $f$ that can be represented as $f = \phi_1 - \phi_2$ for a choice of two \emph{convex} functions. DC functions constitute a very rich class - for instance, they are known to contain all $\mathcal{C}^2$ functions. Such functions have been applied in a variety of contexts including non-convex optimization \cite{yuille2002concave, horst1999dc}, sparse signal recovery \cite{gasso2009recovering} and reinforcement learning \cite{piot2014difference}.

The principal contribution of this paper is a method for piecewise linear regression over the class of DC functions. At the heart of the method is a representation of piecewise linear DC functions via a set of linear constraints, in a manner that generalises the representations used for max-affine regression. The choice of the function is regularised for smoothness by a new seminorm that controls the $\ell_\infty$-Lipschitz constant of the resulting function. The resulting estimate is thus a piecewise linear function, represented as the difference of two piecewise linear convex functions, that are smooth in the sense of having bounded gradients.

The method enjoys two main advantages:
\begin{enumerate}[wide]
    \item It is agnostic to any knowledge about the function, and requires minimal parameter tuning.
    \item It can be implemented efficiently, via quadratic programming, even in high dimensions. For $n$ data points in $\mathbb{R}^d$,  the problems has $2n(2d+1) + 1$ decision variables, and $n^2$ linear constraints, and can be solved in the worst case in  $O(d^2 n^5)$ time by interior-point methods. Furthermore the algorithm does not need to specify partitions for piece-wise linear parts and avoids ad-hoc generalizations of splines or piece-wise linear methods to multi-dimensions.
\end{enumerate} In addition, the method is shown to be statistically viable, in that it is shown to attain vanishing risk as the sample size grows at a non-trivial rate, under the condition that the ground truth has bounded DC seminorm. The risk further adapts to structure such as low dimensional supports. 

Lastly, the solution obtained is a piecewise-linear fit, and thus enjoys interpretability in that features contribution heavily to the value can be readily identified. Further, the fitting procedure naturally imposes $\ell_1$ regularisation on the weights of the piecewise linear parts, thus enforcing a sparsity of local features, which further improves interpretability.

To establish practical viability, we implement the method on a number of regression and classification tasks. The settings explored are those of moderate data sizes - $n \le 10^3$, and dimensions $d \le 10^2$. We note that essesntially all non-parametric models are only viable in these settings - typical computational costs grow with $n$ and become infeasible for large datasets, while for much higher dimensions, the sample complexity - which grows exponentially with $d$ - cause small datasets to be non-informative. More pragmatically, all nonparametric methods we compare against have been evaluated on such data. Within these constraints, the method is shown to have better error performance and fluctuation with respect to popular methodologies such as multivariate adaptive regression splines, nearest neighbour methods, and two-layer perceptrons, evaluated on both synthetic and real world data-sets.

\subsection{Connections to Existing Methodologies}\label{sec:related}

\textbf{Piecewise Linear Regression}  is popular since such regressors can can model the local features of the data without affecting the global fit. In higher than $1$ dimensions, piecewise linear functions are usually fit via choosing a partition of the space and fitting linear functions on each part. The principle difficulty thus lies in choosing these partitions. The approach is usually a rectangular grid - for instance, a variable rectangular partition of the space is studied in \cite{toriello2012fitting} and solved optimally. However the rectangulization becomes prohibitive in high dimension as the number of parts grow exponentially with the dimension. Other approaches include Bayesian methods such as \cite{holmes2001bayesian}, which rely on computing posterior means for the parameters to be fit via MCMC methods.

\textbf{Max-Affine Regression} is a nonparametric approximation to convex regression, originating in \cite{hildreth1954point, holloway1979estimation} that recovers the optimal piece-wise linear approximant to a convex function with the form $f = \max_{i\in [K]} \langle  a_i, \bm x_i \rangle + b_i$ for some $K$. Smoothness of the estimate can be controlled by constraining the convex function to be Lipschitz. The problem is generic in that it is easily argued that piecewise linear convex functions can uniformly approximate any Lipschitz convex function on a bounded domain. Parametric approaches, i.e., with a fixed $K$, are popular, but can be computationally intensive due to the induced combinatorics of which of the $K$ planes is maximised at which data point, and various heuristics and partitioning techniques have to be applied \cite{magnani2009convex, hannah2013multivariate, ghosh2019max}. The nonparametric case, where $K$ grows with $n$, has been extensively analysed in the works \cite{balazs2015near, balazs2016convex}.

On the other hand, if $K = n,$ i.e.~if  the number of affine functions used is the same as the number of datapoints, then the problem becomes amenable to convex programming techniques - when estimating the parameters $a_i, b_i,$ one can remove the nonlinearity induced by the max, and instead enforce the same via $n$ linear constraints. This simple fact allows efficient algorithmic approaches to max-affine functions. The heart of our method for DC function regression is an extension of this trick to DC functions. 

\textbf{Smoothing splines} are an extremely popular regression methodology in low dimensions. The most popular of these are the $L_2$ smoothing splines, which, in one dimension, involve fixing a `knot' at each data point, and estimating the gradients of the function at each point, with regularisation of the form $\int |\hat{f}''|^2$. Unfortunately this $L_2$ regularisation leads to singularities in $d \ge 3$ dimensions, and methods such as thin plate splines generalising these to higher dimensions resort of regularising up to the $d$th derivative of the estimate, leading to an explosion in the number of parameters to be estimated \cite{wahba1990spline}. 

Our method is closer in relation to $L_1$ regularised splines, which in the univariate case regularise $\int |\hat{f}''|$ - it is shown in Appx.~\ref{appx:spline} that in one dimension our method reduces to these. As a consequence, one may view this method as a new generalisation of the $L_1$-spline regressor. 

A number of alternative methods for multivariate splines have been proposed, with several, such as general additive models modelling the data via low dimensional projections and assumptions. The most relevant multivariate spline methods are the adaptive regression splines, originating in \cite{friedman1991multivariate}, which is a greedy procedure for recursively refining a partition of the data, and fitting new polynomials over the parts.

\textbf{Previous DC Modeling} Finally, let us mention that the final chapter of the doctoral thesis of \citet{balazs2016convex} anticipates our study of DC function regression, but gives neither algorithms nor analyses. Subsequently, \citet{cui2018composite} introduces the same DC modeling as us in a broader context, where the loss function can also be a DC function. However their problem ends up being non-convex. They focus on developing a majorization-minimization algorithm to find an approximate solution with desirable guarantees such as convergence to a directional stationary solution.

\section{A brief introduction to DC functions}\label{sec:dc_background}

Difference of convex functions are functions which can be represented as difference of two continuous convex function over a domain $\Omega \subset \mathbb{R}^d$, i.e., the class \begin{align*}
    \mathcal{DC}(\Omega)  \triangleq \{f: \Omega \to \mathbb{R} \,|  \exists &\phi_1,\phi_2 \textrm{ convex, continuous s.t.}\\ &f = \phi_1 - \phi_2 \}.
\end{align*} 
Throughout the text we will set $\Omega = \{ \bm x : \| \bm x \|_\infty \le R\}$. We will assume that the noise and the ground truth function are bounded so that $|\varepsilon|, \sup_{\bm x \in \Omega} |f(\bm x)| \le M.$ As a consequence, $|y| \le 2M$. 

One of the first characterizations of DC functions is due to \citet{hartman1959functions}: a univariate function is a DC function if and only if its directional derivatives each has bounded total variation on all compact subsets of $\Omega$. For higher dimensions it is known that DC functions are a subclass of locally-Lipschitz functions and include $C^{2}$ functions. Therefore, any continuous function can be uniformly approximated by DC functions. For a recent review see \citet{bacakdifference}. 

In the following section we show that D.C functions can fit any sample data. Thus, to allow a bias-variance tradeoff, we regularise the selection of DC functions via the \emph{DC seminorm}
\begin{align*}
    \| f\| \triangleq \inf_{\phi_1,\phi_2}&\sup_{\bm x} \sup_{\bm v_i \in \partial_* \phi_i(\bm x)} \|\bm v_1\|_1 + \|\bm v_2\|_1 \\
    & \textrm{s.t. } \phi_1, \phi_2 \text{ are convex, } \phi_1 - \phi_2 = f,
\end{align*}

where $\partial_*\phi_i$ denotes the set of subgradients of $\phi_i$. The above function is not a norm because every constant function satisfies $\|c\| = 0$. Indeed, if we equate DC functions up to a constant, then the above seminorm turns into a norm. We leave a proof of the fact that the above is a seminorm to Appx.~\ref{appx:seminorm}. Note that the DC seminorm offers strong control on the $\ell_\infty$-Lipschitz constant of the convex parts of at least one DC representation of the function (and in turn on the Lipschitz constant of the function).

We will principally be interested in DC functions with bounded DC seminorm, and thus define \[ \dc_L \triangleq \{ f \in \dc: \|f\| \le L\}.\]

The bulk of this paper concentrates on \emph{piecewise linear DC functions}. This is justified because piecewise linear functions are known to uniformly approximate bounded variation functions, and structural results \cite{ kripfganz1987piecewise, ovchinnikov2002max} showing that \emph{every} piecewise linear function may be represented as difference of two convex piecewise linear functions - i.e., max-affine functions. Since the term is used very often, we will abbreviate ``piecewise linear DC'' as PLDC, and symbolically define \begin{align*}
    \pldc &\triangleq \{ f \in \dc: f\textrm{ is piecewise linear} \},\\
    \pldc_{L} &\triangleq \{ f \in \dc_{L}: f\textrm{ is piecewise linear} \}.
\end{align*}
The following bound on the seminorm of PLDC functions is useful. The proof is obvious, and thus omitted. \begin{proposition}\label{prop:pldc-norm-bound}
Every $f \in \pldc$ can be represented as a difference of two max-affine functions \[f(\bm x) = \max_{k \in [K]} \langle \bm a_k, \bm x \rangle + c_k - \max_{k \in [K]} \langle \bm b_k, \bm x \rangle + c_k'\] for some finite $K$. For such an $f$, $\|f\| \le \max_{k} \|\bm a_k\|_1  + \max_{k} \|\bm b_k\|_1$.
\end{proposition}

\subsection{Expressive Power of Piecewise linear DC functions}

We begin by arguing that PLDC functions can interpolate any finite data. The principle characterisation for DC functions is as follows:
\begin{proposition}\label{prop:dc_interpolate}
For any finite data $\{(\bm x_i, y_i)\}_{i \in [n]},$ there exists a DC function $h: \mathbb{R}^d\rightarrow \mathbb{R}$, that takes values $h(\bm x_i) = y_i$ if and only if there exist $\bm a_i, \bm b_i \in \mathbb{R}^d, z_i \in \mathbb{R}, i\in [n]$ such that:
\begin{align}\label{equ:lem1.2}
\begin{split}
     y_i - y_j + z_i  - z_j&\geq \langle \bm a_j, \bm x_i-\bm x_j\rangle, \quad i,j \in [n]\\
     z_i - z_j  &\geq \langle \bm b_j, \bm x_i-\bm x_j \rangle    , \quad i,j \in [n]
\end{split}
\end{align}

Further, if there exists a DC function that interpolates a given data, then there exists a PLDC function that also interpolates the data.

\begin{proof}
Assuming $h=\phi_1 - \phi_2$ for convex functions $\phi_1$ and $\phi_2$, take $\bm a_j$ and $\bm b_j$ to be sub-gradients of respectively $\phi_1$ and $\phi_2$ then (\ref{equ:lem1.2}) holds by convexity. Conversely, assuming (\ref{equ:lem1.2}) holds, define $h$  as
\begin{align*}\label{equ:max-affine max-affine}
    h(\bm x)&= \max_{i\in [n]} \langle \bm a_i,  \bm x-\bm x_i \rangle + y_i + z_i - \max_{i\in[n]} \langle \bm b_i, \bm x-\bm x_i \rangle + z_i
\end{align*}
$h \in \dc$ since it is expressed as the difference of two max-affine functions. Further, it holds that $h(\bm x_k) = y_k$ for any $k \in [n]$. Indeed, the first condition implies that for any $i \neq k,$ \[ \langle \bm a_i, {\bm x_k - \bm x_i}\rangle + y_i + z_i \le y_k + z_k,\] with equality when $i = k$. Thus, the first maximum simply takes the value $y_k + z_k$ at the input $\bm x_k$. Similarly, the second maximum takes the value $z_k$ at this input, and thus $h(\bm x_k) = (y_k + z_k) - z_k = y_k.$

Notice that the interpolating function given in the above is actually piecewise-linear. Thus, if a DC function fits the given data, then extracting the vectos $\bm a_i, \bm b_i$ and constants $z_i$ as in the first part of the proof, and constructing the interpolant in the second part yields a PLDC function that fits the data.
\end{proof}
\end{proposition}

The principal utility of the conditions stated above is that we can utilise these to enforce the shape restriction of getting a DC estimate in an efficient way when performing empirical risk minimisation. Indeed, suppose we wish to fit a DC function to some data $(\bm x_i, y_i)$. We may then introduce decision variables $\hat{y}_i, z_i, \bm a_i, \bm b_i,$ where the $\hat{y}_i$ represent the value of our fit at the various $\bm x_i,$ and then enforce the linear constraints of the above proposition (with $y_i$ replaced by $\hat{y}_i$) while minimising a loss of the form $\sum (y_i - \hat{y}_i)^2$. Since these constraints are linear, the resulting program is thus convex, and can be efficiently implemented. This observation forms the core of our algorithmic proposal in \S\ref{sec:algo}

The above characterisation relies on existence of vectors that may serve as subgradients for the two convex functions. This condition can be removed, as in the following.
\begin{proposition}\label{thm:interpolate_pldc} Given any finite data $\{(\bm x_i, y_i)\}_{i \in [n]},$ such that $y_i \neq y_j \implies \bm x_i \neq \bm x_j$, there exists a PLDC function which interpolates this data.
\end{proposition}
\begin{proof}
The interpolating function is constructed by adding and subtracting a quadratic function to the data. Let \[C \triangleq \max_{i,j} \frac{|y_i-y_j|}{\|\bm x_i-\bm x_j\|_2^2}.\] Then the piecewise linear function
\begin{align}\label{equ:max-affine max-affine}
    h(\bm x)&\triangleq  \max_{i \in [n]} \langle C \bm x_i ,\bm x -\bm x_i  \rangle +\frac{1}{2}C \|\bm x_i\|^2 + \frac{1}{2}y_i \nonumber \\
    &- \max_{i \in [n]} \langle C \bm x_i ,\bm x  -\bm x_i \rangle +\frac{1}{2}C\|\bm x_i\|^2 - \frac{1}{2}y_i
\end{align} satisfies the requirements. Indeed, the function is DC, since it the difference of two max-affine funcitons. The argument proceeds similarly to the previous case - at any $\bm x_j,$ we have:
\begin{align*}
 &\max_{i \in [n]}\langle C \bm x_i ,\bm x_j -\bm x_i  \rangle +\frac{1}{2}C \|\bm x_i\|^2 + \frac{1}{2}y_i \\
 &= \max_{i \in [n]}   \frac{1}{2}C\|\bm x_j\|^2 -\frac{C}{2} \| \bm x_i - \bm x_j \|^2+ \frac{1}{2}y_i \\
 &\leq \max_{i \in [n]} \frac{1}{2}C\|\bm x_j\|^2 - \frac{|y_i- y_j|}{2} + \frac{1}{2}y_i  \\
 &\leq \frac{1}{2}C\|\bm x_j\|^2 + \frac{1}{2}y_j.
\end{align*}
 However the upper bound is reached by choosing $i=j$ in the max operator. Similarly the value of the second convex function at $\bm x_j$ is equal to $\frac{1}{2}C\|\bm x_j\|^2  -\frac{1}{2}y_j$ which together results in $h(\bm x_j) = y_j$. Note that $h(\bm x)$ has the $\ell_\infty$-Lipschitz constant $2 C\max_{i,j} \|\bm x_i -\bm x_j\|_1,$ and $\|h \| \le 2C \max_i \|\bm x_i\|_1.$ 
\end{proof}

Next, in order to contextualise the expressiveness of DC functions, we argue that the popular parametric class of ReLU neural networks can be represented by PLDC functions, and vice versa. This is also argued in \cite{cui2018composite} for a $2$ layer network.
\begin{proposition}\label{prop:relu_subset_pldc}
A fully connected neural network $f$, with ReLU activations, and $D$ layers with weight matrices $\bm W^1, \dots, \bm W^D$, i.e,
\begin{align*}
    f(\bm x) &=  \sum_j w^{D+1}_j a^D_j \\
    a^{l+1}_i &= \max (\sum_{j} w^{l+1}_{i,j}  a^{l}_j, 0), \quad D > l\geq 1\\
    a^1_i &= \max (\sum_{j}w^1_{i,j} x_j, 0),
\end{align*}
is a PLDC function with the DC seminorm bounded as
\begin{align*}
    \|f\| \leq |\bm w^{D+1}|^T \Big( \prod_{l=1}^D |\bm W^l |\Big) \vec{\bm 1},
\end{align*}
\end{proposition}
where $|\cdot|$ is the entry-wise absolute value. The above is proved in Appx.~\ref{appx:relu_pldc} via an induction over the layers using the relations
\begin{align*}
    &\max(\max(a,0)\!-\!\max(b,0),0) = \max(a,b,0)\! -\! \max(b,0)\\
    &\max(a,b) + \max(c,d) = \max (a+b,a+d, b+c, b+d).
\end{align*}
\begin{proposition}\label{prop:pldc_subset_relu}
Every PLDC function with $K$ hyper-planes can be represented by a ReLU net with $2\lceil \log_2 K \rceil$ layers and maximum width of $8K$.
\end{proposition}
The proof is constructed in Appx.~\ref{appx:relu_pldc} using the relations
\begin{align*}
    &\max(a,b,c,d) = \max \left(\max (a,b), \max(c,d)\right)  \\
    &\max(a,b)  = \max(a\!-\!b,0) + \max(b,0) -\max(-b,0) .
\end{align*}

\section{Algorithms}\label{sec:algo}

We begin by motivating the algorithmic approach we take. This is followed by separate section developing key portions of the algorithm.  

Suppose we observe a data-set $S_n = \{ (\bm x_i, y_i)\}_{i\in[n]}$ generated iid from some distribution $\mathcal{X}\times\mathcal{Y}$ and a convex loss function $\ell:\mathbb{R}\times\mathbb{R}\rightarrow \mathbb{R}_+$ bounded by $c$. The goal is to minimize the expected risk
\begin{align}\label{erm:dc2}
     &\min_{f\in\mathcal{DC}} \mathbb{E} [(f(\bm x) - y)^2]
\end{align}
by choosing an appropriate function $f$ from the class of $\mathcal{DC}$ functions. Note instead of the squared error, the above could be generalised to any bounded, Lipschitz losses $\ell(f(\bm x), y)$. Note also that the squared loss is bounded in our setting because of our assumption that both the ground truth and noise are bounded.

There are two basic problems with the above - the distribution is unknown, so the objective above cannot be evaluated, and that the class of DC functions is far too rich, and so the problem is strongly underdetermined. In addition, directly optimising over all DC functions is an intractable problem.

To begin with, we reduce the problem by instead finding the values that a best fit DC function must take at the datapoint $\bm x_i$, and then fitting a PLDC functions with convex parts that are max-affine over precisely $n$ linear functionals on this. This has the significant advantage of reducing the optimisation problem to a set of convex constraints. Quantitatively, this step is justified in \S\ref{sec:analysis}, which argues that the error induced by this approximation via PLDC functions is dominated by the intrinsic risk of the problem.

To handle the lack of knowledge of the distribution, we resort to uniform generalisations bounds in the literature. Our approach to relies on the following result, which mildly generalises the bounds of \citet{bartlett2002rademacher}, and uniformly controls the generalisation error of an empirical estimate of the expectation (specialised to our context):

\begin{theorem}\label{thm:srm_bound} Let $\{(\bm x_i, y_i)\}_{i \in [n]},$ be i.i.d.~data, with $n$ assumed to be even. Let the \emph{empirical maximum discrepancy} of the class $\dc_L$, be defined as, \[ \hat{D}_n( \mathcal{\dc}_{L}) \triangleq \sup_{f \in \dc_{L}} \frac{2}{n}\left( \sum_{i = 1}^{n/2} f(\bm x_i) - \sum_{i = n/2 + 1}^n f(\bm x_i)\right).\] With probability $\ge 1-\delta$ over the data, it holds holds uniformly over \emph{all $f \in \dc \cap \{|f| \le M\}$} that 
\begin{align}\label{gen_bound_BM}
&\left| \mathbb{E}[(f(\bm x)- y)^2] - \frac{1}{n}\sum_{i=1}^n (f(\bm x_i) - y_i)^2 \right|  \notag \\\le&\,\,  12M\hat D_n (\dc_{2\|f\| + 2}) \notag \\ &\,\,\, + 45M^2\sqrt{\frac{C \max(2, \log_2 \| f\|) +\ln(1/\delta)\big)}{n}},
\end{align} where $C$ is a constant independent of $f, M ,R$.
\end{theorem}

The above statement essentially arises from a doubling trick over a Rademacher complexity based bound for a fixed $\|f\|$. The broad idea is that since $\dc_L \subset \dc_{L'}$ for $L \le L',$ we can separately develop Rademacher complexity based bounds over $L$ of the form $2^j$, each having the more stringent high-probability requirement $\delta_j = \delta 2^{-j}$. A union bound over these then extends these bounds to all of $\dc = \bigcup_{j \ge 1} \dc_{2^j},$ and for a particular $f$, the bound for $j = \lceil \log_2 \|f\| \rceil$ can be used. See \S\ref{appx:pf_of_srm_bound} for details.

Optimising the empirical upper bound on $\mathbb{E}[(f(X) - Y)^2]$ implied by the above directly leads to a structural risk minimization over the choice of $L$. The crucial ingredient in the practicality of this is that for DC functions, $\hat{D}_n(\dc_L) = L \hat{D}_n(\dc_1)$, and further, $\hat{D}_n(\dc_1)$ can be computed via linear programming. Thus, the term $\hat{D}_n$ above serves as a natural, efficiently computable penalty function, and acts exactly as a regularisation on the DC seminorm.

\subsection{Computing empirical maximum discrepancy.}

Throughout this and the following sections, we use $\hat{y}_i$ to denote $\hat{f}(\bm x_i)$, where $\hat{f}$ is the estimated function.

The principle construction relies on the characterisation of Proposition \ref{prop:dc_interpolate}. 

\begin{theorem}
    Given data $\{(\bm x_i, y_i)\},$ the following convex program computes $\hat{D}_n(\dc_L)$ 
    \begin{align}\label{program:d_hat_dc}
        &\qquad \max_{\hat{y}_i, z_i, \bm a_i, \bm b_i} \frac{2}{n} \left( \sum_{i = 1}^{n/2} \hat{y}_i - \sum_{i = n/2+1}^n \hat{y}_i \right) \\
        &\textrm{s.t.} \begin{cases}
     \hat y_i - \hat y_j + z_i  - z_j\geq \langle \bm a_j, \bm x_i-\bm x_j\rangle & i,j \in [n]  \quad \mathrm{(i)}\\
    z_i - z_j  \geq \langle \bm b_j, \bm x_i-\bm x_j \rangle    &  i,j \in [n] \quad \mathrm{(ii)}\\
    \|\bm a_i\|_1 + \|\bm b_i\|_1 \leq L & \quad i\in[n] \quad \mathrm{(iii)}
    \end{cases}\notag
    \end{align}
    Further, $\hat{D}_n(\dc_L) = L \hat{D}_n(\dc_1)$.
    \begin{proof}
        By Proposition \ref{prop:dc_interpolate}, conditions $(\mathrm{i})$ and $(\mathrm{ii})$ are necessary and sufficient for the existence of a DC function that takes the values $\hat{y}_i$ at $\bm x_i$. Thus, these first two conditions allow exploration over all values a DC function can take. Further, by the second part of Proposition \ref{prop:dc_interpolate} if a DC function interpolates this data, then so does a PLDC function, with $\bm a_i$ and $\bm b_i$ serving as the gradients of the max-affine parts of the function. Thus, by Proposition \ref{prop:pldc-norm-bound}, the condition $(\mathrm{iii})$ is necessary and sufficient for the DC function implied by the first to conditions to have seminorm bounded by $L$. It follows that the conditions allow exploration of all values a $\dc_L$ function may take at the given $\{\bm x_i\}$, at which point the claim follows.
        
        Now, notice that if we multiply each of the decision variables in the above program by $L$, the value of the program is multiplied by a factor of $L$, while the constraints $\mathrm{(i), (ii)}$ remain unchanged. On the other hand, the constraint $\mathrm{(iii)}$ is modified to $\|\bm a_i \| + \|\bm b_i \| \le 1$. Thus, the resulting program is $L$ times the program computing $\hat{D}_n(\dc_1)$, ergo $\hat{D}_n(\dc_L) = L\hat{D}_n(\dc_1)$.
    \end{proof}
\end{theorem}

\subsection{Structural Risk Minimisation}

To perform the SRM, we again utilize the structural result of Proposition $1$ to determine the values that the optimal estimate takes at each of the $\bm x_i$. The choice of the values is penalised by the seminorm as $\lambda L,$ where $\lambda = 24M\hat{D}_n(\dc_1),$ which may be computed using the program (\ref{program:d_hat_dc}). Note that the logarithmic term in the generalisation bound (\ref{gen_bound_BM}) is typically small, and is thus omitted in the following. This also has the benefit of rendering the objective function convex, as it avoids the $\log L$ term that would instead enter. If desired, an convex upper bound may be obtained for the same, for instance by noting that $\sqrt{ \max(1, \log_2 \|f\|)} \le 1 + \|f\|$. This has the effect of bumping up the value of $\lambda$ required by $O(M^2/\sqrt{n})$. However, theoretically this term is strongly dominated by the $\hat{D}_n$ (see \S\ref{sec:analysis}), while practically even the value $\lambda = \hat{D}_n(\dc_1) $ tends to produce overly smoothened solutions (see \S\ref{sec:expt}). 

With the appropriate choice of $\lambda,$ this yields the following convex optimisation problem, 
    \begin{align}\label{program:SRM}
        &\qquad \min_{\hat{y}_i, z_i, \bm a_i, \bm b_i, L} \sum_{i = 1}^n ( \hat{y_i} - y_i)^2 + \lambda L\\
        &\textrm{s.t.} \begin{cases}
    \hat y_i - \hat y_j + z_i  - z_j\geq \langle \bm a_j, \bm x_i-\bm x_j\rangle & i,j \in [n]  \\
    z_i - z_j  \geq \langle \bm b_j, \bm x_i-\bm x_j \rangle    &  i,j \in [n] \\
    \|\bm a_i\|_1 + \|\bm b_i\|_1 \leq L & \quad i\in[n] 
    \end{cases}\notag
    \end{align}
    
    Once again, in the above constraints, the first two are necessary and sufficient to ensure that a DC function taking the values $\hat{y}_i$ at $\bm x_i$ exists, with the vectors $\bm a_i, \bm b_i$ serving as witnesses for the subgradients of the convex parts of such a function at the $\bm x_i$, and the third constraint enforces that the function has seminorm at most $L$. Notice that the third condition effectively imposes $\ell_1$-regularisation on the weight vectors $\bm a_i, \bm b_i$. This causes these weights to be sparse.
    
    Finally, we may use the witnessing values $\hat{y}_i, z_i$ and $\bm a_i, \bm b_i,$ to construct, in the manner of Proposition \ref{prop:dc_interpolate}, the following PLDC function, which interpolates the values $\hat{y}_i$ to the entirety of the domain. Notice that since this function has the same loss as any DC function that satisfies $f(\bm x_i) = y_i$, this construction enjoys risk bounds constructed above.    \begin{equation}\label{eqn:pldc_estimate} \begin{split} \hat{f}(\bm x) \triangleq & \max_{i\in [n]} \langle \bm a_i,  \bm x-\bm x_i \rangle + \hat y_i + z_i \\ &\quad -\max_{i\in[n]} \langle \bm b_i, \bm x-\bm x_i \rangle + z_i \end{split}\end{equation}

    \subsection{Computational Complexity}
    
    \textbf{Training} First, we note that we may replace the constraints on the $1$-norms of the vectors $\bm a_i, \bm b_i$ in the above by linear constraints via the standard trick of writing the positive and negative parts of each of their components separately. Overall, this renders the program (\ref{program:d_hat_dc}) as a linear programming problem over $2n(2d + 1)$ variables, and with $n^2$ non-trivial constraints. Note that in our setting, one typically requires that $n \ge d$ - that one has more samples than dimensions. Via interior point methods, this problem may be solved in $O(n^5)$ time. 
    
    For the least squares loss $\ell(y,\hat{y}) = (y - \hat{y})^2,$ the second program (\ref{program:SRM}) is a convex quadratic program when the $1$-norm constraints are rewritten as above. The decision variables are the same as the first problem, with the addition of the single $L$ variable, and the constraints remain identical. Again, via interior point methods, these programs can be solved in $O(d^2n^5)$ time (see Ch. 11 of \citet{boyd2004convex}). The latter term dominates this complexity analysis. We note that in practice, these problems can be solved significantly faster than the above bounds suggest.
    
    \textbf{Speeding up training via a GPU implementation} an iterative solver for a modified version of the SRM in program~(\ref{program:SRM}) is given in Algorithm~(\ref{program:SRM_ADMM}) in the Appx~\ref{appx:admm_algor} via the ADMM method \cite{boyd2011distributed}. Each iteration of this algorithm can be distributed to $O(n^2+nd)$ parallel processors. A python implementation is given in our GitHub repository~\ref{link:github} which is compatible with GPU's. We note that a similar algorithm for that of Lipschitz convex regression is provided in \cite{balazs2016convex, mazumder2019computational} however not all the ADMM blocks are solved in closed form and require additional optimization in each iteration.
    
    \textbf{Prediction} By appending a $1$ to the input, and the constants $y_i + z_i - \langle \bm a_i , \bm x_i \rangle $ and $z_i - \langle \bm b_i, \bm x_i\rangle$ to $\bm a_i$ and $\bm b_i$, we can reduce the inference time problem to solving two maximum inner product search problems over $n$ vectors in $\mathbb{R}^{d+1}$. This is a well developed and fast primitive, e.g.~\citet{mips} provide a locality sensitive hashing based scheme that solves the problem in time that is sublinear in $n$.

\section{Analysis}\label{sec:analysis}
We note that this analysis makes extensive use of the work of \citet{balazs2015near, balazs2016convex} on convex and max-affine regression, with emphasis on the latter thesis, which contains certain refinements of the former paper. 

In this section, we assume that the true model $y = f(\bm x) + \varepsilon_i$ holds for a $f$ that is a DC function, and that we have explicit knowledge that $\| f\| \le L.$ Also recall our assumption that the distribution is supported on a set that is contained in the $\ell_\infty$ ball of radius $R$. We begin with a few preliminaries

\paragraph{A lower bound on risk} The minimax risk of estimation under the squared loss is $\Omega(n^{-2/d+2}).$ This follows by setting $p = 1$ (i.e., Lipschitzness) in Theorem 3.2 of \cite{gyorfi2006distribution}, which can be done since the standard constructions of obstructions to estimators used in showing such lower bounds all have regularly varying derivatives, and thus finite DC seminorms. 

\paragraph{PLDC solutions are not lossy} Lemma 5.2 of \cite{balazs2016convex} argues that for every convex $L$-Lipschitz functions $\phi$ with Lipschitz constant\footnote{\citet{balazs2016convex} presents an argument with the $2$-norm of subgradients bounded, but this can be easily modified to the case of bounded $1$-norm under bounds on $\|\bm x\|_\infty$.}, $\sup_x \|\nabla_* \phi\|_1$ bounded by $L$, there exists a Lipschitz max-affine function $\phi_{\mathrm{PL}}$ with maximum over $n$ pieces such that $\| \phi - \phi_{\mathrm{PL}}\| \le 36LR n^{-2/d}$. Recall that PLDC functions can be represented as differences of max-affine functions, and DC functions as differences of convex functions. Since the DC seminorm controls the $1$-norm of the subgradients, which dominates the $2$-norm, it follows that for every $\dc_L$ function $f$, there exists a $\pldc_{2L}$ function $f_{\mathrm{PL}}$ with $\| f - f_{\mathrm{PL}}\|_\infty\! =\! O(n^{-2/d}).$ Note that the resulting excess risk in the squared error due to using PLDC functions can thus be bounded as $O(n^{-4/d}),$ which is $o(n^{-2/d+2}),$ i.e., it is dominated by minimax risk.

\subsection{Statistical risk}

The bound (\ref{gen_bound_BM}) provides a instance specific control on the generalisation error of an estimate via the empirical maximum discrepancy $\hat{D}_n$. This section is devoted to providing generic bounds on the same under the assumption of i.i.d.~sampled data. We adapt the analysis of \cite{balazs2016convex} for convex regression in order to do so. The principal result is as follows. 

\begin{theorem} For distributions supported on a compact domain $\Omega \subset \{\|\bm x\|_\infty \le R\},$ with $n\ge d$, it holds that
    \[ \hat{D}_n(\dc_L) \le 60LR \left(\frac{d}{n}\right)^{2/d+4} \left( 1 + 2\frac{\log(n/d)}{d+4}\right).\] Further, if the ground truth $f \in \dc_L$, then with probability at least $1-\delta$ over the data, the estimator $\hat{f}$ of $(\ref{eqn:pldc_estimate})$ satisfies \begin{align*} &\mathbb{E}[ |Y - \hat{f}(\bm x)|^2] \le \mathbb{E}[|Y- f(\bm x)|^2]\\ &\quad + O((n/d)^{-2/d+4} \log(n/d) ) + O(\sqrt{\log(1/\delta)/n}).\end{align*}
\end{theorem}
\begin{proof}
Assume $f \in \dc_L$. Note that for any convex representation $f = \psi_1 - \psi_2$, we may instead construct a representation $f = \phi_1 - \phi_2 + c$ for a constant $c$ so that the resulting convex function $\phi_1$ and $\phi_2$ are uniformly bounded on the domain - indeed, this may be done by setting $c = f(0)$, and $\phi_k = \psi_k - \psi_k(0)$. The $\phi$s retain the bound on their Lipschitz constants, and thus are uniformly bounded by $LR$ over the domain. Thus, we may represent the class of $\dc_L$ functions as the sum of a constant, and a $\dc_L$ function whose convex parts are bounded. Call the latter class $\dc_{L,0}.$ Importantly, since the constants are cancelled in the computation of empirical maximum discrepancy, we can observe that \( \hat{D}_n(\dc_L) = \hat{D}_n(\dc_{L,0}).\)

The principle advantage of the above exercise is that the empirical discrepancy for DC functions with bounded convex parts can be controlled via the metric entropy of bounded Lipschitz convex functions, which have been extensively analysed by \citet{balazs2016convex}. This is summarised in the following pair of lemmata. The first argues that the discrepancy of $\dc_{L,0}$ functions is controlled by that of bounded Lipschitz convex functions. 
\begin{lemma}\label{lem:discrepancy_dc_to_c}
    Let $\mathcal{C}_{L,LR}$ be the set of convex functions that are $L$-Lipschitz and bounded by $LR$. Then \( \hat{D}_n(\dc_{L,0}) \le 2 \hat{D}_n(\mathcal{C}_{L,LR}).\)
\end{lemma}
The proof of the above is left to Appx~C. The second lemma, due to Dudley, is a generic method to allow control on the discrepancy. We state this for $\mathcal{C}_{L, LR}$ \begin{lemma}\label{lem:dudley}
    Let $\mathcal{H}_\infty(\mathcal{C}_{L, LR}, \epsilon)$ be the metric entropy of the class $\mathcal{C}_{L, LR}$ under the sup-metric $d(f,g) = \|f-g\|_\infty$. Then the empirical maximum discrepancy is bounded as \[ \hat{D}_n(\mathcal{C}_{L, LR}) \le \inf_{\epsilon > 0} \left( \epsilon + LR \sqrt{2\frac{\mathcal{H}_\infty(\mathcal{C}_{L, LR}, \epsilon)}{n}}  \right).\]
\end{lemma}

Finally, we invoke the control on metric entropy of bounded Lipschitz convex functions provided by \cite{balazs2015near, balazs2016convex}

\begin{namthm*}[\cite{balazs2015near, balazs2016convex}] For $\epsilon \in (0, 60LR],$  \[ \mathcal{H}_\infty(\mathcal{C}_{L,LR}, \epsilon) \le 3d\left(\frac{40 L R}{\epsilon}\right)^{d/2}\log \left( \frac{60LR}{\epsilon} \right). \]  
\end{namthm*}

Using the above in the bound of Lemma \ref{lem:dudley}, and choosing $\epsilon = (60 LR) (d/n)^{2/d+4}$ yields the claim. 

Control on the excess risk of the estimator follows readily. For any $\lambda \ge 24M \hat{D}_n(\dc_1),$ we have have, with high probability,\begin{align*}
    &\mathbb{E}[ (\hat{f} - Y)^2 ] \\
    \le\,\, &\hat{\mathbb{E}}[ (\hat{f} - Y)^2] + \lambda\|\hat{f}\| + 2\lambda + O(1/\sqrt{n})\\
    \le\,\, &\hat{\mathbb{E}}[ (f - Y)^2] + \lambda \| f\| + 2\lambda + O(1/\sqrt{n})\\
    \le\,\, &\mathbb{E}[ (f-Y)^2] + 2\lambda \|f\|  + 2\lambda + O(1/\sqrt{n}),\\
    \le\,\, &\mathbb{E}[ (f-Y)^2] + 2\lambda (L+1) + O(1/\sqrt{n})
\end{align*}
where the first and last inequalities utilise (\ref{gen_bound_BM}), while the second inequality is because $\hat{f}$ is the SRM solution. The claim follows on incorporating the bound on $\hat{D}_n$ developed above, and since we choose $\lambda$ proportional to the same.\vspace{-30pt}
\end{proof}
\textcolor{black}{\paragraph{On adaptivity} Notice that the argument for showing the excess risk bound proceeds by controlling $\hat{D}_n$. This allows the argument to adapt to the dimensionality of the data. Indeed, if the true law is supported on some low dimensional manifold, then the empirical discrepancy, which only depends on the observed data, grows only with this lower dimension. More formally, due to the empirical discrepancy being an empirical object, we can replace the metric entropy control over DC functions in $\mathbb{R}^d$ by metric entropy of DC functions supported on the manifold in which the data lies, which in turn grows only with the (doubling) dimension of this manifold. }

\paragraph{On suboptimality} Comparing to the lower bound, we find that the above rate is close to being minimax, although the (multiplicative) gap is about $n^{4/d^2}$ and diverges polynomially with $n$. In analogy with the observations of \citet{balazs2016convex} for max-affine regression, we suspect that this statistical suboptimality can be ameliorated by restricting the PLDC estimate to have some (precisely chosen) $K_n < n$ hyperplanes instead of the full $n$. However, as discussed in \S\ref{sec:related}, such restrictions lead to increase in the computational costs of training, and thus, we do not pursue them here. 

\section{Experiments}\label{sec:expt}
In this section we apply our method to both synthetic and real datasets for regression and multi-class classification. The datasets were chosen to fit in the regime of $n\le 10^3, d \le 10^2$ as described in the introduction. All results are averaged over $100$ runs and are reported with the $95\%$ confidence interval. 

For the DC function fitting procedure, we note that that the theoretical value for the regularization weight tends to over-smooth the estimators. This behaviour is expected since the bound (\ref{gen_bound_BM}) is designed for the worst case. For all the subsequent experiments we make two modifications - since none of the values in the datasets observed are very large, we simply set $12M = 1,$ and further, we choose the weight, i.e. $\lambda$ in (\ref{program:SRM}), by cross validation over the set $2^{-j} \hat{D}_n(\dc_1)$ for $j \in [-8:1]$. Fig.~\ref{fig:2dplot} presents both these estimates in a simple setting where one can visually observe the improved fit. Note that this tuning is still minimal - the empirical discrepancy of $\dc_1$ fixes a rough upper bound on the $\lambda$ necessary, and we explore only a few different scales.

\begin{figure}[htb] 
\centering
\includegraphics[width = 0.25\textwidth]{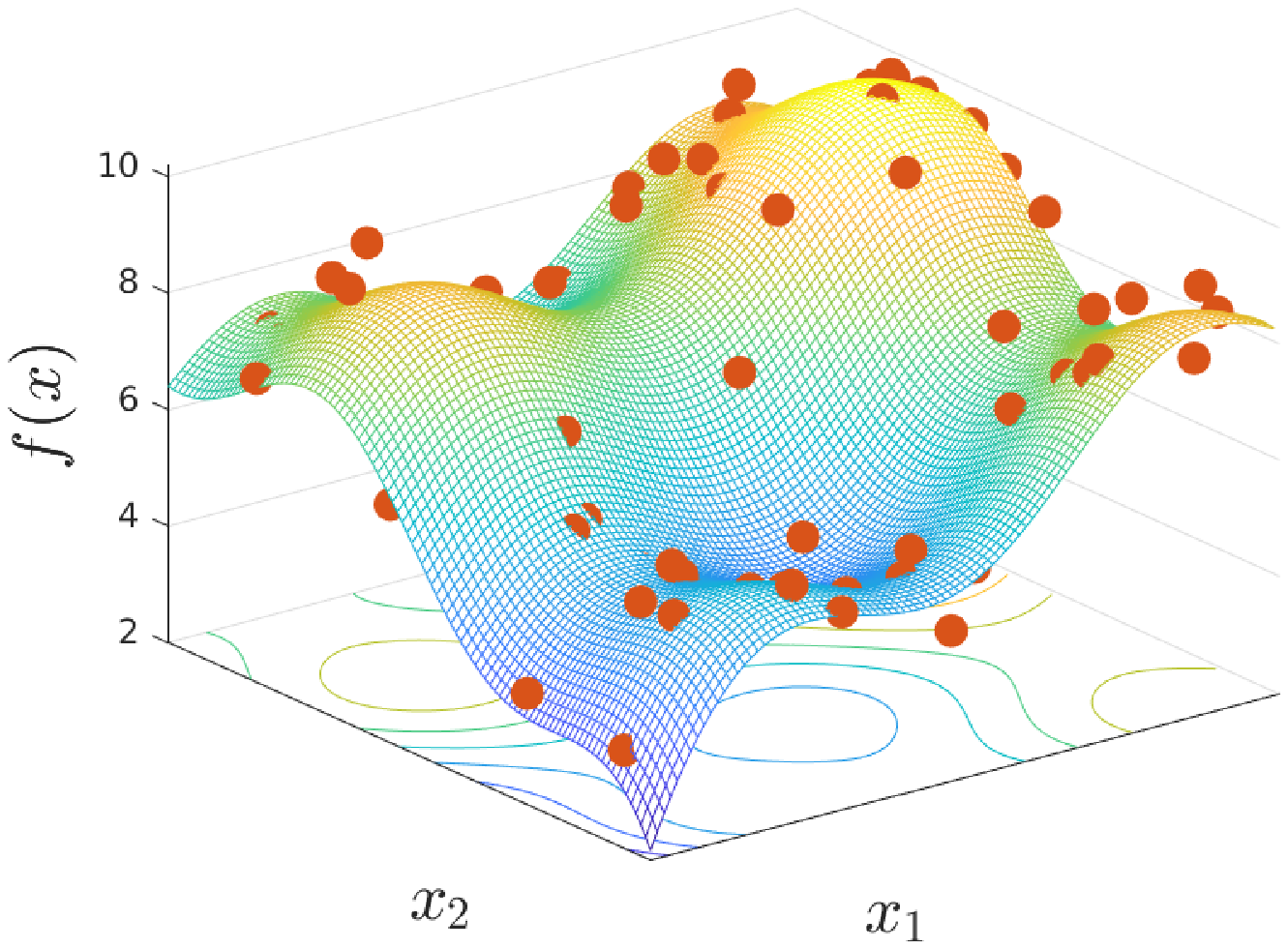}\\
\includegraphics[width = 0.24\textwidth]{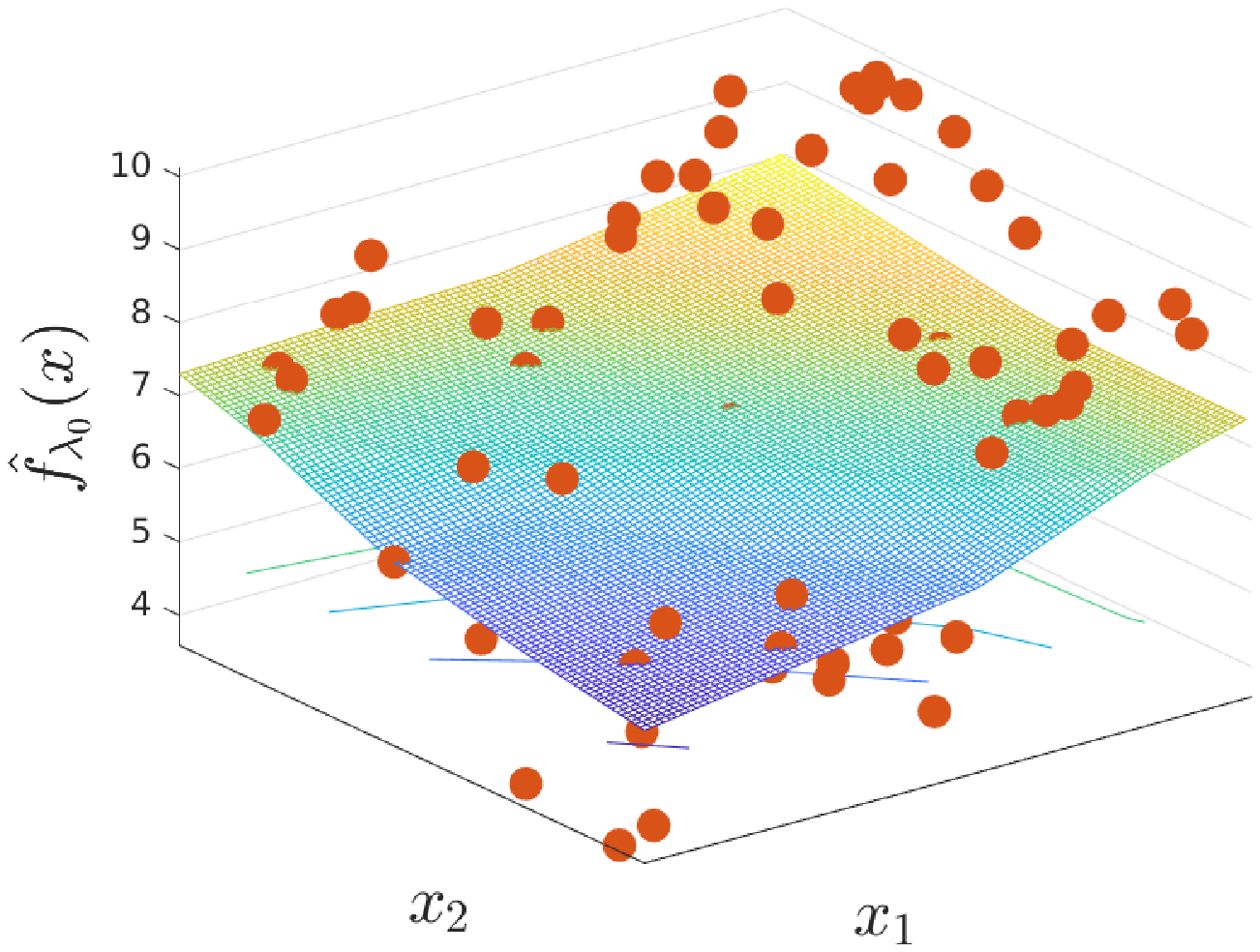}~
\includegraphics[width = 0.24\textwidth]{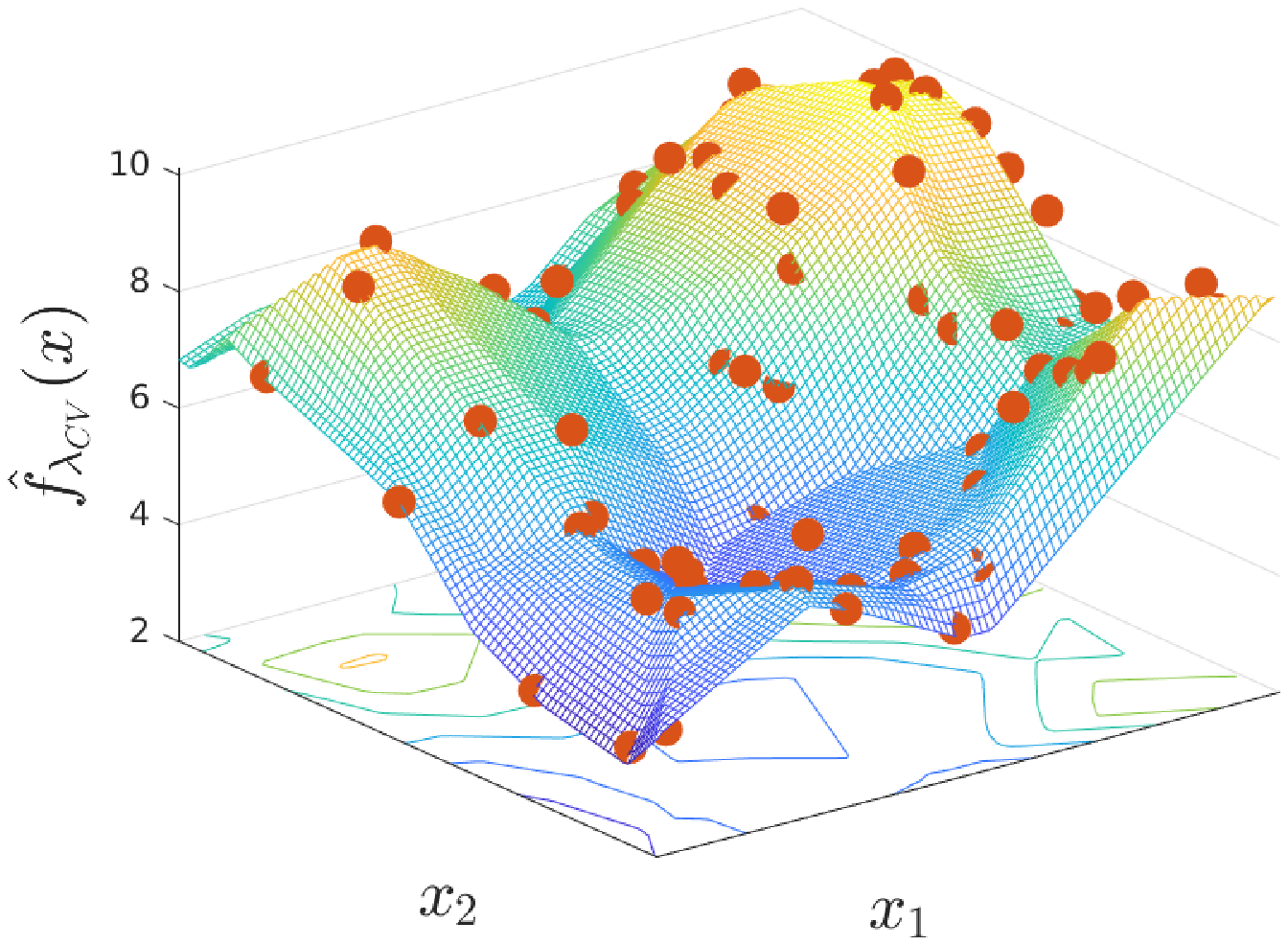}
\caption{\emph{Top} A two dimensional function along with the sampled points used for estimating the function; \emph{Bottom} learned DC function via L1 regression using only $\lambda= 2\hat D_n$ \emph{(left)}; using cross validation over $\lambda$ \emph{(right)}.}\label{fig:2dplot}
\end{figure}

For the \textbf{regression task} we use the $L_1$ empirical loss in our algorithm, instead of $L_2$. That is, the objective in (\ref{program:SRM}) is replaced by $\sum |y_i - \hat{y}_i|$. This change allows us to implement the fitting program as a linear programming problem and significantly speeds up computation. However, in the following we will only report the $L_2$ error of the solutions obtained thi way. We compare our method to a multi-layer perceptron (neural network) with variable number of hidden layers chosen from $1:10$ by $5$-fold cross validation, Multivariate Adaptive Regression Splines (MARS) and $K$-nearest neighbour($K$-NN) where Best value of $K$ was chosen by $5$-fold cross validation from $1:10$. 

For the multi-class \textbf{classification task} we adopt the multi-class hinge loss to train our model, i.e.~the loss
\begin{align*}
    \sum_i \sum_{j\neq y_i}^m\max( f_j(\bm x_i)-f_{y_i}(\bm x_i) +1, 0),
\end{align*}
where $m$ is the number of classes and $f_j$'s are DC functions. We compare with the same MLP as above but trained with the cross entropy Loss, KNN and a one versus all SVM. 

In both cases, we have used MATLAB Statistics and Machine learning Library for their implementation of MLP, KNN and SVM. For MARS we used the open source implementation in ARESLab toolbox implemented in MATLAB. Our code along with the other algorithms is available in our GitHub repository\footnote{\href{https://github.com/Siahkamari/Piecewise-linear-regression}{https://github.com/Siahkamari/Piecewise-linear-regression}\label{link:github}}\\.  

In each of the following tasks, we observe that our method performs competitively to all considered alternatives in almost every dataset, and often outperforms them, across the variation in dimensionality and dataset sizes. 

\paragraph{Regression on Synthetic Data}
We generated data from the function, 
\begin{align*}
y &= f(\bm x) + \varepsilon\\
    f(\bm x) &= \sin\big(\frac{\pi}{\sqrt{d}}\sum_{j=1}^d x_{j}\big)+\big(\frac{1}{\sqrt{d}}\sum_{j=1}^d x_{j}\big)^2,
\end{align*}
where the $\bm x$ is sampled from a standard Gaussian, while $\varepsilon$ is a centred Gaussian noise with standard deviation of $0.25$.

We generate $50$ points for training. For testing we estimate the Mean Squared Error based on a test set of $5000$ points without the added noise. We normalize the MSE by variance of the values of test data and multiply by $100$. Results are presented in Fig.~\ref{fig:synthetic}. Observe that our algorithm consistently outperforms the competing methods, especially as we increase the dimension of the data. Furthermore our algorithm has lower error variance across the runs. 
\begin{figure}[!htb] 
\includegraphics[width=8cm]{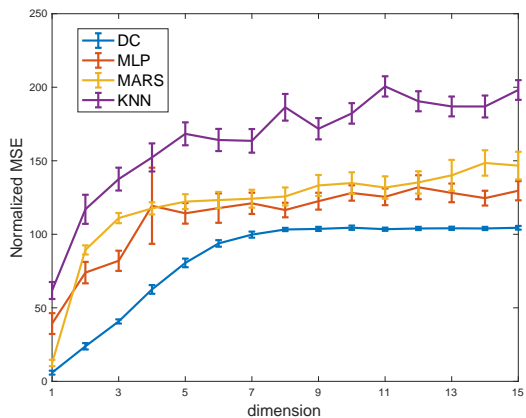}
\centering
\caption{Mean Squared Error in a regression task vs dimension of the data in the synthetic experiements. Note that both the value and the size of the error bars are consistently better than the competing methods}\label{fig:synthetic} \vspace{-10pt}
\end{figure}

\paragraph{Regression on Real Data}
We apply the stated methods to various moderately sized regression datasets that are available in the MATLAB statistical machine learning library. The results are presented in Fig.~\ref{fig:real_regression}. 

In the plot, the datasets are arranged so that the dimension increases from left to right. We observe that we do comparably to the other methods for some datasets and outperform in others. See Appx.~D for a description of each of the datasets studied.
\begin{figure}[htb] 
\includegraphics[width=8cm]{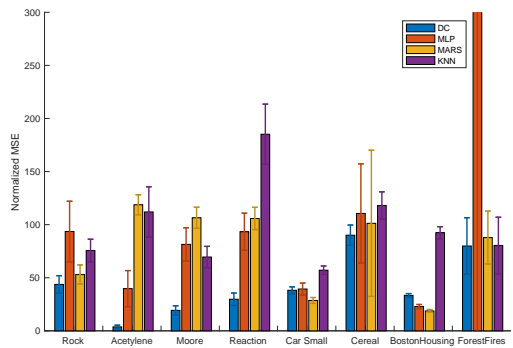}
\centering
\caption{Normalized Mean Squared Error in regression tasks.}\label{fig:real_regression} \vspace{-15pt}
\end{figure}

\paragraph{Multi-class classification}
We used popular UCI classification datasets for testing our classification algorithm. We repeated the experiments $100$ times We present the mean and $95\%$ C.I.s on a 2-fold random cross validation set, in Fig.~\ref{fig:real_class}. We observe to perform comparably to other algorithms on some datasets and outperform in others. \begin{figure}[!htb] 
\includegraphics[width=8cm]{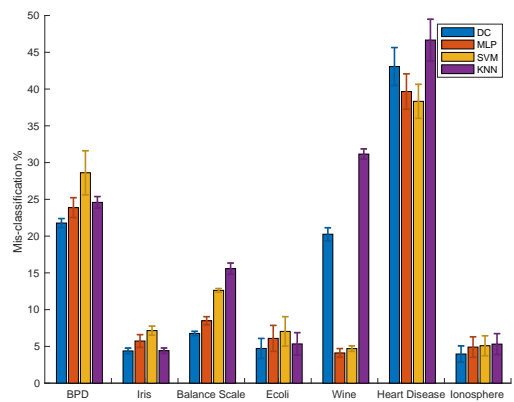}
\centering
\caption{Miss-Classification on UCI data sets.}\label{fig:real_class} \vspace{-10pt}
\end{figure}

\section{Discussion}
The paper proposes an algorithm to learn piecewise linear functions using difference of max-affine functions. Our model results in linear or convex programs which can be solved efficiently even in high dimensions. We have shown several theoretical results on expressivity of our model, as well as its statistical properties, including good risk decay and adaptivity to low dimensional structure, and have demonstrated practicality of our resulting model under regression and classification tasks.

\paragraph{Wider context} 
Non-parametric methods are most often utilised in settings with limited data in moderate dimensions. Within this context, along with strong accuracy, it is often desired that the method is fast, and is interpretable, especially in relatively large dimensions.

In settings with large datasets, accuracy is relatively easy to achieve, and speed at inference is more important. 
In low dimensional settings, this makes methods such as MARS attractive, due to their fast inference time, while in high-dimensions MLPs are practically preferred. In settings with low dimensionality and small datasets, interpretability and speed take a backseat due to the small number of features, while accuracy become the critical requirement, promoting use of kernelised or nearest neighbour methods. 

On the other hand, for small datasets in moderate to high dimensions, interpretability gains an increased emphasis. Within this context, our method results in a piecewise linear fit for which it is easy to characterise the locally important features, and further does so with relatively sparse weights via the $\ell_1$ regularisation on $\bm a_i, \bm b_i$. Further, since the suboptimality of our statistical risk  is controlled as $n^{O(1/d^2)},$ as the dimension increases, our method gets closer to the optimal accuracy. This thus represents the natural niche where application of DC function regression is appropriate.
\section{Acknowledgment}
This research was supported by NSF CAREER Award 1559558,
CCF-2007350 (VS), CCF-2022446 (VS), CCF-1955981 (VS) and the Data Science Faculty Fellowship from the Rafik B. Hariri Institute.
\bibliography{biblo}
\bibliographystyle{icml2020}

\newpage

\newpage
\onecolumn

\newpage
\onecolumn
\begin{appendix}
\thispagestyle{empty}
\begin{center}{ \textbf{\LARGE Appendix to `Piecewise Linear Regression via a Difference of Convex Functions'}}\end{center}

\section{Connection to $L_1$-regularised splines} \label{appx:spline}

The following proposition argues that in the univariate case fitting a DC function is equivalent to fitting an $L_1$-regularised spline. Note that in the bivariate case, $L_1$ splines are regularised by the Frobenius $1$-norm of the Hessian of the matrix, while in the multivariate case it is similarly regularised by a higher derivative tensor. Thus, our method is an alternate generalisation of the $L_1$-spline in the univaraite case.

\begin{proposition} In the univariate setting, solving regression with the $L_1$-spline objective,
\begin{align}\label{prog:spline1}
    \inf_{f\in C^2} \sum_{i=1}^n \ell(f( x_i), y_i) + \lambda \int_0^1 |f^{''}(s)| d s.
\end{align}
is equivalent to regression with difference of convex functions penalized by their DC seminorm, \emph{with a free linear term}
\begin{align}\label{prog:dc1}
    \min_{\phi_1, \phi_2 \ \text{convex} , a} &\sum_{i=1}^n \ell(\phi_1(x_i) - \phi_2(x_i) + a x   ,y_i) + 2\lambda \sup_{x\in [0, 1]} | \phi_1'(x)| + | \phi_2' ( x)| 
\end{align}
\end{proposition}
\begin{proof}
Suppose $f$ is the solution to (\ref{prog:spline1}) and $\phi_1 - \phi_2 + ax  = f$, where $\phi_1$ and $\phi_2$ are convex. Hence for some $a_1, a_2\in \mathbb{R}$: 
\begin{align*}
    \phi_1'(x) &= a_1 + \int_{0}^{x}  f''^+(s) + g(s) \,ds  \\
    \phi_2'(x) &= a_2 + \int_{0}^{x}  f''^-(s) + g(s)\, ds  \\
    a &= a_2 - a_1 + f'(0),
\end{align*}
where $g(s) \geq 0$ otherwise the second differential of $\phi_1(s)$ or $\phi_2(s)$ would be negative which contradicts convexity. Since the error term in (\ref{prog:dc1}) is invariant to choices of $a_1, a_2$ and $g$, the minimization in (\ref{prog:dc1}) seeks to minimize only the regularization term. Thus, 
\begin{align*}
&\min_{\substack{ \phi_1, \phi_2 \textrm{ convex} \\ \phi_1 - \phi_2 +ax= f}}\sup_x \bigg(|\phi_1'(x)| + |\phi_2'(x)|\bigg)\\ 
=\,\,&\min_{g\geq 0}\min_{a_1,a_2}\sup_x \bigg( \Big|a_1+\int_{0}^{x} f''^+(s)+g(s)ds\Big| + \Big|a_2+\int_{0}^{x} f''^-(s)+g(s)ds\Big|\bigg) \\
\overset{*}{=}\,\, &\min_{g\geq 0}\min_{a_1,a_2} \max\bigg(\Big|a_1\Big| + \Big|a_2\Big|,\ \Big|a_1+\int_{0}^{1} f''^+(s)+g(s)ds\Big| + \Big|a_2+\int_{0}^{1} f''^-(s)+g(s)ds\Big|\bigg)\\
\overset{**}{=}\,\,&\min_{g\geq 0}\bigg(\frac{1}{2}\Big|\int_{0}^{1} f''^+(s)+g(s)ds\Big|  + \frac{1}{2}\Big|\int_{0}^{1} f''^-(s)+g(s)ds\Big| \bigg)\\
=\,\, &\min_{g\geq 0} \bigg(\frac{1}{2}\int_{0}^{1} |f''(s)| ds +  \int_{0}^{1} g(s) ds\bigg)\\
=\,\,& \frac{1}{2}\int_{0}^{1} |f''(s)| ds.
\end{align*}
Therefore (\ref{prog:spline1}) and (\ref{prog:dc1}) are equivalent as they have the same objective.

$(*)$ is true since the expressions inside the absolute value are convex and monotonic in terms of $x$, which causes the inner maximization to take place at endpoints of $x$, i.e. $x=0$ or $x=1$. To show $(**)$ consider,
\begin{align*}
    c_1 &= \int_{0}^{1} f''^+(s)+g(s)ds \\
    c_2 &= \int_{0}^{1} f''^-(s)+g(s)ds.
\end{align*}
It's sufficient to show
\begin{align}\label{equ:CC}
   C = \min_{a_1, a_2} \max\Big (|a_1| + |a_2|, |a_1+c_1| + |a_2+c_2|\Big) = \frac{|c_1| + |c_2|}{2}.
\end{align}
by choosing $a_1 = -\frac{c_1}{2}$ and $a_2 = -\frac{c_2}{2}$ the value on the right hand side of (\ref{equ:CC}) is achieved. Thus $C \leq \frac{|c_1| + |c_2|}{2} $. To show this upper bound is tight, assume $C < \frac{|c_1| + |c_2|}{2}$, hence both options for the max player should be smaller than this, e.g:
\begin{align*}
    |a_1| + |a_2| < \frac{|c_1| + |c_2|}{2}.
\end{align*}
However this in turn causes: 
\begin{align*}
    C \geq |a_1+c_1| + |a_2+c_2| &\geq |c_1| - |a_1| + |c_2| - |a_2| \\
    &> \frac{|c_1| + |c_2|}{2},
\end{align*}
which contradicts the assumption $C < \frac{|c_1| + |c_2|}{2}$. 
\end{proof}


\section{Proofs Omitted from \S\ref{sec:dc_background}}

\subsection{Proof that the DC seminorm is indeed a seminorm}\label{appx:seminorm}
  
\begin{proposition}
The functional \[ \|f\| := \min_{\substack{ \phi_1, \phi_2 \textrm{ convex} \\ \phi_1 - \phi_2 = f} } \sup_x \sup_{\bm v_i \in \partial_* \phi_i(\bm x)} \|\bm v_1 \|_1 + \| \bm v_2\|_1 \] is a seminorm over the class $\mathcal{DC}(\mathbb{R}^d)$. Further, if DC functions are identified up to a constant, it is a norm.
\begin{proof} $ $ \newline
\emph{Homogenity:} Since non-negative scalings of convex functions are convex, if $f = \phi_1 - \phi_2$, then for any real $a$, $af = a\phi_1 - a\phi_2$ is a DC representation if $a > 0$, and $f = -a\phi_2 - (-a\phi_1)$ is a DC representation if $a \le 0.$ Trivially, if $\bm v \in \partial_* \phi_i$ then $a\bm v \in \partial_* (a\phi_i)$ and since $\| \cdot\|_1$ is a norm, it holds that $\|a \bm v\|_1 = |a| \|\bm v\|_1$. Thus, \begin{align*}\|a f\| &= \min_{\substack{ \phi_1, \phi_2 \textrm{ convex} \\ \phi_1 - \phi_2 = af} } \sup_x \sup_{\bm v_i \in \partial_* \phi_i(\bm x)} \|\bm v_1 \|_1 + \| \bm v_2\|_1 \\
&= \min_{\substack{ \phi_1, \phi_2 \textrm{ convex} \\ \phi_1 - \phi_2 = f} } \sup_x \sup_{\bm v_i \in \partial_* \phi_i(\bm x)} \|(\pm a) \bm v_1 \|_1 + \| (\pm a) \bm v_2\|_1 \\
&= \min_{\substack{ \phi_1, \phi_2 \textrm{ convex} \\ \phi_1 - \phi_2 = f} } \sup_x  \sup_{\bm v_i \in \partial_* \phi_i(\bm x)} |a| ( \|\bm v_1 \|_1 + \| \bm v_2 \|_1 )\\
&= |a| \|f\|.
\end{align*}

\emph{Triangle Inequality:} For a DC function $f$, let $\mathscr{C}_f$ be the set of pairs of continuous convex functions whose difference is $f$, i.e. \[ \mathscr{C}_f := \{ (\phi_1, \phi_2) \textrm{ cont., convex }: \phi_1 - \phi_2 = f \}.\]

We first argue that for DC functions $f,g$,  \begin{align*} \mathscr{C}_{f+g} = \{ (\psi_1, \psi_2): \exists (\phi_1, \phi_2) \in \mathscr{C}_f, (\widetilde{\phi}_1, \widetilde{\phi}_2) \in \mathscr{C}_g  \textrm{ s.t. } \psi_1 = \phi_1 + \widetilde{\phi}_1, \psi_2 = \phi_2 + \widetilde{\phi}_2  \} \end{align*}

This follows because any decomposition $\psi_1 - \psi_2 = f+g,$ and $\phi_1 - \phi_2 = f$ gives decomposition $(\psi_1 + \phi_2, \psi_2 + \phi_1) \in \mathscr{C}_g$ and vice versa, so $\mathscr{C}_{f+g}$ lies within the right hand side above. On the other hand, for any pair $(\phi_1, \phi_2) \in \mathscr{C}_f$ and $(\widetilde{\phi}_1, \widetilde{\phi}_2) \in \mathscr{C}_g,$ clearly $(\phi_1 + \widetilde{\phi}_1, \phi_2 + \widetilde{\phi}_2) \in \mathscr{C}_{f+g},$ so the right hand side lies within the $\mathscr{C}_{f+g}$. 

As a consequence, \begin{align*}
    &\,\,\,\,\|f + g\|\\
    &=\min_{ (\psi_1, \psi_2) \in \mathscr{C}_{f+g}  } \sup_x \sup_{\bm u_i \in \partial_* \psi_i(\bm x)} \|\bm u_1 \|_1 + \| \bm u_2\|_1 \\
              &=\min_{ \substack{(\phi_1, \phi_2) \in \mathscr{C}_f \\ (\widetilde{\phi}_1, \widetilde{\phi}_2) \in \mathscr{C}_g } } \sup_x  \sup_{\substack{\bm v_i \in \partial_* \phi_i(\bm x)\\ \widetilde{{\bm v}}_i \in \partial_* \tilde{\phi}_i(\bm x)}}\|  \bm v_1 +  \widetilde{\bm v}_1\|_1 +\|  \bm v_2 + \widetilde{\bm v}_2\|_1\\
              &\le \min_{ \substack{(\phi_1, \phi_2) \in \mathscr{C}_f \\ (\widetilde{\phi}_1, \widetilde{\phi}_2) \in \mathscr{C}_g } } \sup_x \sup_{\substack{\bm v_i \in \partial_* \phi_i(\bm x)\\ \widetilde{{\bm v}}_i \in \partial_* \tilde{\phi}_i(\bm x)}}\|  \bm v_1\|_1 + \| \widetilde{\bm v}_1\|_1 +\|  \bm v_2 \|_1 + \| \widetilde{\bm v}_2\|_1 \\
              &\le \min_{(\phi_1, \phi_2) \in \mathscr{C}_f} \sup_x \sup_{\substack{\bm v_i \in \partial_* \phi_i(\bm x)}}\|  \bm v_1 \|_1 + \| \bm v_2\|_1 +  \min_{(\widetilde{\phi}_1, \widetilde{\phi}_2) \in \mathscr{C}_g} \sup_x \sup_{\substack{ \widetilde{{\bm v}}_i \in \partial_* \tilde{\phi}_i(\bm x)}} \|\widetilde{\bm v}_1\|_1 + \|\widetilde{\bm v}_2\|_1 \\
              &= \|f\| + \|g\|.
\end{align*}

\emph{Positive Definite up to constants:} We note that if $\|f\| = 0,$ then $f$ is a constant function. Indeed, the norm being zero implies that there exists a DC representation $f = \phi_1 - \phi_2$ such that for all $\bm x$, the biggest subgradients in $\partial_* \phi_i(\bm x)$ have $0$ norm, ergo $\|\nabla \phi_1 (x) \|_1 = \|\nabla \phi_2(x)\|_1 = 0,$ and thus, $\nabla f = 0$ everywhere. Conversely, clearly for every constant function, $\|c\|= 0$. Thus, the norm is zero iff the function is a constant. 
Lastly, note that for any $f$ and a constant $c$, $\|f\| = \|f + c\|.$ Thus, equating DC functions that differ only by a constant turns the above seminorm into a norm.  
\end{proof}
\end{proposition}

\subsection{ReLU networks and PLDC functions}\label{appx:relu_pldc}
\subsubsection{Proof of Proposition \ref{prop:relu_subset_pldc}}
\begin{proof}

Recall that a fully connected neural network with ReLU activations, and $D$ layers with weight matrices $\bm W^1, \dots, \bm W^D$, is a function, here $f$, taking the form
\begin{align*}
    f(\bm x) &=  \sum_j w^{D+1}_j a^D_j \\
    a^{l+1}_i &= \max (\sum_{j} w^{l+1}_{i,j}  a^{l}_j, 0), \quad D > l\geq 1\\
    a^1_i &= \max (\sum_{j}w^1_{i,j} x_j, 0),
\end{align*}

Our proof of the statement is constructed by induction over the layers of the network, using the following relations:
\begin{align}
    \max(a,b) + \max(c,d) &= \max (a+b,a+d, b+c, b+d). \label{equ:sum_max} \\
    \max(\max(a,0)-\max(b,0),0) &= \max(a,b,0) - \max(b,0) \nonumber\\
    &=\max_{z\in\{0,1\}, t\in\{0,1\}}(tz a +t(1\!-\!z)b) - \max_{t\in\{0,1\}} tb. \label{equ:max_max}
\end{align}
We will prove each node $a^{l}_i$ is a DC function of the form:
\begin{align*}
    a_i^l &= \max_{k\in[K_l]} \langle \bm\alpha^l_{i,k}, \bm x \rangle - \max_{k\in[K_l]} \langle \bm\beta^l_{i,k}, \bm x \rangle
\end{align*}
such that 
\begin{align*}
    \max_k(\|\bm\alpha^{l+1}_{i,k}\| , \|\bm \beta^{l+1}_{i,k}\|) \leq  \sum_{j} |w_{i,j}| \max_k(\|\bm\alpha^l_{j,k}\| , \|\bm\beta^l_{j,k}\|)
\end{align*}
The base of induction is trivial by looking at the definition for $a_i^1$ which is max of linear terms and therefore convex.
Now we assume the claim holds $a_j^{l}$ for all $j$, we induce the results for $a_i^{l+1}$. Assume layer $l$ has $h$ hidden units. For clarity we drop the index $l$ from $w^{l+1}_j$, $a^l_j$ and $K_l$. We further define $s \triangleq a_i^{l+1}$.
\begin{align*}
    s &= \max\bigg(\sum_{j\in[h]} w_{j} a_j, 0\bigg)\\
    &= \max \bigg(\sum_{j\in[h]} w_{j}^+ a_j - \sum_{j\in[h]} -w_{j} ^- a_j, 0 \bigg)\\
    &= \max \bigg(\sum_{j\in[h]}  \max_{k\in[K]} \Big \langle w_{j}^+ \bm\alpha_{j,k}, \bm x \Big \rangle + \max_{k\in[K]} \Big \langle -w_{j}^-\bm \beta_{j,k}, \bm x \Big \rangle - \sum_{j\in[h]} \max_{k\in[K]}\Big  \langle -w_{j}^- \bm \alpha_{j,k}, \bm x\Big  \rangle - \max_{k\in[K]} \Big \langle w_{j}^+ \bm\beta_{j,k}, \bm x \Big \rangle , 0 \bigg)\\
    &\overset{(\ref{equ:sum_max})}{=} \max \bigg( \max_{\bm {k1},\bm {k2}\in[h]^K} \Big  \langle \sum_{j\in[h]} w_j^+\bm \alpha_{j,\bm {k1} (j)}-w_{j}^- \bm\beta_{j,\bm {k2} (j)}, \bm x \Big  \rangle -\max_{\bm {k1},\bm {k2}\in[h]^K} \Big  \langle \sum_{j\in[h]} -w_j^-\bm \alpha_{j,\bm {k1}(j)}+w_{j}^+ \bm\beta_{j,\bm {k2}(j)}, \bm x \Big \rangle ,0\bigg)\\
    &\overset{(\ref{equ:max_max})}{=}  \max_{\bm {k1},\bm {k2}\in[h]^K,z\in\{0,1\},t\in\{0,1\}} \Big \langle t \sum_{j\in[h]} z(w_j^+ \bm \alpha_{j,\bm {k1}(j)}- w_{j}^-   \bm\beta_{j,\bm {k2}(j)})  + (1-z)(-w_j^-\bm \alpha_{j,\bm {k1}(j)}+w_{j}^+\bm \beta_{j,\bm {k2}(j)}), \bm x \Big  \rangle \\
    &\qquad -\max_{\bm {k1},\bm {k2}\in[h]^K,t\in\{0,1\}} \Big   \langle t\sum_{j\in[h]} -w_j^-\bm \alpha_{j,\bm {k1} (j)}+w_{j}^+ \bm\beta_{j,\bm {k2} (j)}, \bm x \Big  \rangle 
\end{align*}
Now consider $z=1$ and note that for each term of $max$ function:
\begin{align*}
    \|\sum_{j\in[h]} w_j^+  \bm\alpha_{j,\bm {k1}(j)}- w_{j}^-   \bm\beta_{j,\bm {k2}(j)}\|  &\leq \sum_{j\in[h]} |w_j| \max (\|\bm\alpha_{j,\bm {k1}(j)}\|, \|\bm\beta_{j,\bm {k2}(j)}\|)\\
    &\leq \sum_{j\in[h]} |w_j| \max_{k} (\|\bm\alpha_{j,k}\|, \|\bm\beta_{j,k}\|)
\end{align*}
The proof is analogous for $z=0$.
\end{proof}

\subsubsection{Proof of Proposition \ref{prop:pldc_subset_relu} }

\begin{proof}
        
        For convenience, we extend the input dimension by $1$ and append a constant $1$ to the end, so that the effecive inputs are $(\bm x, 1).$ This allows us to avoid writing constants in the following.

        Let the PLDC function in question be \[ f(\bm x) = \max_{k \in [K]} \langle \bm a_k, \bm x \rangle  - \max_{k \in [K]} \langle \bm b_k, \bm x\rangle .\]
        
        We give a construction below to compute $\max_{k \in [K]} \langle \bm a_k , \bm x\rangle$ via a ReLU net that has at most $2\lceil \log_2 K\rceil $ hidden layers, and width of at most $\le 2K$ \emph{if $K$ is a power of $2$}. Note that repeating this construction parallelly with the $\bm a$s replaced by $\bm b$s would then constitute the required construction.
        
        By adding dummy values of $\bm a_k, \bm b_k$ (i.e., equal to $0$), we may increase the $K$ to a power of two without affecting the depth in the claim. The resulting width will be bounded as $2 \cdot 2^{\lceil \log_2 K \rceil} \le 4K,$ and will yield the bound claimed. This avoids us having to introduce further dummy nodes later, and simplifies the argument.
        
        To begin with, we set the first hidden layer to be of $2K$ nodes, labelled $(k,+)$ and $(k,-)$ for $k \in [K]$, with the outputs \begin{align*} h^1_{k,+} &= \max(\langle \bm a_k ,\bm x\rangle, 0) \\ h^1_{k,-} &= \max(\langle -\bm a_k ,\bm x\rangle, 0)\end{align*}
        
        Thus, the outputs of the first layer encode the inner products $\langle \bm a_k, \bm x\rangle$ in pairs. In the next layer, we we collate two $\pm$ pairs of nodes into $3$ nodes, indexed as $(j,0), (j,1), (j,0)$ for $j \in [K]$ such that the outputs are \begin{align*}
            h^2_{j,0} &= \max( h^1_{2j-1, +} - h^{1}_{2j-1, -} -h^1_{2j, +} + h^1_{2j, -} , 0) = \max( \langle \bm a_{2j-1} - \bm a_{2j}, \bm x \rangle, 0)\\
            h^2_{j,1} &= \max( h^1_{2j, +} - h^1_{2j, -},0) = \max(\langle \bm a_{2j} , \bm x\rangle,0) \\
            h^2_{j,1} &= \max( -h^1_{2j, +} + h^1_{2j, -},0) = \max(\langle -\bm a_{2j} , \bm x\rangle,0)
        \end{align*}
        
        In the next hidden layer, we may merge these three nodes for each $j$ into two, giving a total width of $\le K + 1,$ represented as $(\ell, +)$ and $(\ell, -)$. This utilises the simple relation \[ \max(a,b) = \max(a-b,0) + \max(b,0) - \max(-b,0), \] easily shown by some casework. Let us define the outputs \begin{align*}
            h^3_{\ell, +} &= \max(h^2_{\ell, 0} + h^2_{\ell, 1} - h^2_{\ell, 2}, 0) = \max( \max(\langle \bm a_{2\ell - 1}, \bm x \rangle, \langle \bm a_{2\ell}, \bm x \rangle), 0)  \\
            h^3_{\ell, -} &= \max(-h^2_{\ell, 0} - h^2_{\ell, 1} + h^2_{\ell, 2}, 0) = \max(- \max( \langle \bm a_{2\ell - 1}, \bm x \rangle,  \langle \bm a_{2\ell}, \bm x \rangle), 0)
        \end{align*}
        
        But now note that the outputs of the third hidden layer are analogous to those of the first hidden layer, but with some of the maximisations merged, and thus the circuit above can be iterated. This is exploiting the iterative property that $\max(a,b,c,d) = \max( \max(a,b), \max(c,d))$. Thus, we may apply this construction $\log_2K$ times in total, reducing the width of every odd numbered hidden layered by $2$ each time, and using two hidden layers for each step. At the final step, by the iterative property of $\max$, we will have access to a 2-node hidden layer with outputs $\max( \max_{k \in [K]} \langle \bm a_k , \bm x \rangle , 0 )$ and $\max( -\max_{k \in [K]} \langle \bm a_k , \bm x \rangle , 0 ),$ and the final layer may simply report the difference of these, with no nonlinearity.
        
        We note that more efficient constructions are possible, but are harder to explain. In particular, the above construction can be condensed into one with $\le 3K$ width and $\lceil \log_2 K \rceil$ depth.
\end{proof}


\section{Proof of Theorem \ref{thm:srm_bound}}\label{appx:pf_of_srm_bound}

For succinctness, we let $\dc_L^M := \dc_L \cap \{ \sup_{ \bm x \in \Omega} |f(\bm x) \le M\}$. We note that this section makes extensive use of the work of \citet{bartlett2002rademacher}, and assumes some familiarity with the arguments presented there. We begin a restricted version of Theorem \ref{thm:srm_bound}.

\begin{lemma}\label{lem:srm_tech_basic}
    Given $n$ i.i.d.~ samples $S = \{ (\bm x_i, y_i)\}$ from data $y_i = f(\bm x_i) + \varepsilon_i$, where both $f$ and $\varepsilon$ are bounded by $M$, with probability at least $1-\delta,$ it holds uniformly over all $f \in \dc_L^M$ that \[ |\mathbb{E}[  (f(\bm x)- y)^2] - \hat{\mathbb{E}}_S[ (f(\bm x)- y)^2]| + 12M  \hat{D}_n(\dc_L)  + 45M^2 \sqrt{\frac{4 + 8 \log(2/\delta)}{n}} \]
\end{lemma}

We prove this lemma after the main text, and first show how this is utilised to produce the bound of Theorem \ref{thm:srm_bound}. The key insight is that the classes $\dc_L^M$ are nested, in that $\dc_L^M \subset \dc_{L'}^M$ for any $L \le L'$. This allows a doubling trick over the $L$s to bootstrap the above bound to all of $\dc^M = \bigcup_{L \ge 1} \dc_L^M$. The technique is standard, see, e.g., Thm. 26.14 of \cite{shalev2014understanding}.

For naturals $j \ge 1,$ let $L_j = 2^j$ and $\delta_j = \delta 2^{-j} = \delta/L_j$. With probability at least $1-\delta_j,$ it holds for all $f \in \dc_{L_j}^M\setminus \dc_{L_{j-1}}^M$ simultaneously that \begin{align*}
    |\mathbb{E}[  (f(\bm x)- y)^2] - \hat{\mathbb{E}}_S[ (f(\bm x)- y)^2]| \le 12M  \hat{D}_n(\dc_{L_j})  + 45M^2 \sqrt{\frac{4 + 8\log L_j + 8 \log(2/\delta)}{n}}
\end{align*} 

since $\sum \delta_j = \delta,$ the union bound then extends this bound to all DC functions, whilst allowing us to use the appropriate value of $j$ as desired. Taking $j_f = \max(1,\lceil \log_2  \| f\| \rceil)$, we find that for any $f \in \dc^M,$ \begin{align*} &| \mathbb{E}[  (f(\bm x)- y)^2] - \hat{\mathbb{E}}_S[ (f(\bm x)- y)^2] | \\ &\le  12M \hat{D}_n(\dc_{2^{j_f}})  + 45M^2 \sqrt{\frac{4 + 8\log 2^{j_f} + 8 \log(2/\delta)}{n}} \\ &\le  12M \hat{D}_n(\dc_{2\|f\| + 2})  + 45M^2 \sqrt{\frac{4 + 8 \ln(2) \max( 1, \log_2 \|f\|) + 8 \log(2/\delta)}{n}} \end{align*}

\begin{proof}[Proof of Lemma \ref{lem:srm_tech_basic}]
        
We first recall Theorem 8 of \cite{bartlett2002rademacher}, which says that if $\ell$ is a loss function that is uniformly bounded by $B$, then for any class of functions $\mathcal{F}$ and an iid sample $S = \{ (\bm x_i, y_i)\}$ of size $n$, with probability at least $1-\delta/2$, it holds uniformly over all $f \in \mathcal{F}$ that \[ \mathbb{E}[ \ell( f(\bm x), y)] \le \hat{\mathbb{E}}_S[ \ell(f(\bm x), y)] + \Re_n( \ell \circ \mathcal{F}) + B \sqrt{\frac{8 \log(4/\delta)}{n}},\] where $\Re_n(\ell \circ \mathcal{F})$ is the Rademacher complexity of the class $\{ \ell(f(\bm x) , y) - \ell (0,y)\}$. 

In our case we are interested in $\mathcal{F} = \dc_L^M$ and $\ell = (\cdot - \cdot)^2$. Since $|f| \le M$ and $|y| \le 2M,$ we may take $B = 9M^2$.  

Next, we use properties of Theorem 12, part 4 of \cite{bartlett2002rademacher}, which says that for $L$-Lipschitz $\ell$ which takes the value $\ell(0,0) = 0,$ $\Re_n( \ell \circ \mathcal{F}) \le 2L \Re_n(\mathcal{F})$. Since our assumptions enforce that the argument to the squared loss is bounded by $3M,$ we note that the loss is $6M$-Lipschitz on this class. Thus, we conclude that with $n$ samples, uniformly over all $f \in \dc_L^M$ \[ \mathbb{E}[  (f(\bm x)- y)^2] \le \hat{\mathbb{E}}_S[ (f(\bm x)- y)^2] + 12M \Re_n( \dc_L^M) + 9M^2 \sqrt{\frac{8 \log(4/\delta)}{n}}.\]

Next we utilise Lemma 3 of \cite{bartlett2002rademacher}, and that $\dc_L^M \subset \dc_L$ to conclude that with probability at least $1-\delta/2$, \[ \Re_n(\dc_L^M) \le \hat{D}_n(\dc_L^M) + 3M \sqrt{\frac{2}{n}} \le \hat{D}_n(\dc_L) + 3M \sqrt{\frac{ 4 +  8 \log(4/\delta))}{n}}. \]

The claim then follows by using the union bound.
\end{proof}

\section{Proof of Lemma \ref{lem:discrepancy_dc_to_c}}\label{appx:pf_disc_lemma}

\begin{proof}
        For any function $f$, let \[ \Delta(f) \triangleq \sum_{i= 1}^{n/2} f(\bm x_i) - \sum_{i = n/2+1}^{n} f(\bm x_i).\]
        Note that $\Delta(f+g) = \Delta(f) + \Delta(g)$. Then
        \begin{align*} 
            \hat{D}_n(\dc_{L,0}) &= \sup_{f \in \dc_{L,0}} \Delta(f) \\
            &= \sup_{\substack{f \in \dc_{L,0} \\ \phi_1,\phi_2 \in \mathcal{C}_{L,LR} \\ \phi_1 - \phi_2 = f}} \Delta(\phi_1) - \Delta(\phi_2) \\
            &\le \sup_{\phi_1,\phi_2 \in \mathcal{C}_{L,LR}} \Delta(\phi_1) - \Delta(\phi_2) \\
            &\le \sup_{\phi_1 \in \mathcal{C}_{L,LR}} \Delta(\phi_1) + \sup_{\phi_2 \in \mathcal{C}_{L,LR}} \Delta(\phi_2) \\
            &= 2\hat{D}_n(\mathcal{C}_{L,LR})\end{align*}
        where the first inequality holds because of the representation of $\dc_{L,0}$ as a difference of $\mathcal{C}_{L,LR}$ functions, as discussed in the main text, and the second inequality holds since the supremum is non-negative, as $0 \in \mathcal{C}_{L,LR}$.
    \end{proof}

\newpage
\section{Algorithms via ADMM}\label{appx:admm_algor}
\vspace{-0.2cm}
This section provides a parallel ADMM optimizer for the piece-wise linear regression problem.
\begin{algorithm*}[h!]
   \caption{Piecewise Linear Regression via a Difference of Convex Functions}
   \label{program:SRM_ADMM}
\begin{algorithmic}
   \STATE {\bfseries Input:} data $\{(\bm x_i, y_i)\,|\, \bm x_i \in \mathbb R^D, y_i \in \mathbb R\}_{i\in[n]},\, \rho = 0.01, \, \lambda, \, T$
   \STATE {\bfseries Initialize:} 
   \begin{align*}
     \hat y_i &= z_i  = s_{i,j} = t_{i,j} = \alpha_{i,j} = \beta_{i,j} = 0  & i,j\in[n]\\
    \bm L &=\bm a_i =\bm b_i = \bm p_i= \bm q_i = \bm u_{i} = \bm \gamma_{i} = \bm \eta_i = \bm \zeta_i = \bm 0_{D\times 1}, &  i\in[n]\\
    \Lambda_i &= \Big(n\bm x_i\bm x_i^T - \bm x_i (\sum_i \bm x_i)^T  - (\sum_i \bm x_i)\bm x_i^T + \sum_j\bm x_j\bm x_j^T + I \Big)^{-1}  &  i\in[n]
   \end{align*}
   \FOR{$t=1$ {\bfseries to} $T$}
   \STATE \begin{align*}
       A_i &=\sum_j\alpha_{j,i} - \alpha_{i,j} + s_{j,i} - s_{i,j} + \langle \bm a_i + \bm a_j, \bm x_i-\bm x_j\rangle + 2y_j & i \in [n]\\
        B_i &= \sum_j\beta_{j,i} - \beta_{i,j}+ t_{j,i} - t_{i,j} + \langle \bm b_i + \bm b_j, \bm x_i-\bm x_j\rangle & i\in [n]\\
        \hat y_i & = \frac{2}{2+n\rho}y_i + \frac{\rho}{2(2+n\rho)} A_i -  \frac{\rho}{2(2+n\rho)} B_i & i\in [n]\\
      z_i & = -\frac{1}{2+n\rho}y_i + \frac{1}{2n(2+n\rho)} A_i +  \frac{1+ n\rho}{2n(2+n\rho)} B_i & i\in [n]\\
         \bm a_i &= \Lambda_i( \bm p_{i} - \bm \eta_{i}+\sum_j (\alpha_{i,j} + s_{i,j} + \hat y_i -  \hat y_j + z_i - z_j)(\bm x_i - \bm x_j) ) & i\in [n] \\
    \bm b_i &=\Lambda_i( \bm q_{i} - \bm \zeta_{i}+\sum_j (\beta_{i,j} + t_{i,j} + z_i - z_j)(\bm x_i - \bm x_j) ) & i\in [n] \\
    L_d&= \frac{1}{n}(-\lambda_d/\rho + \sum_i\gamma_{i,d} + |p_{i,d}|+ |q_{i,d}|+u_{i,d}) & d\in[D] \\
    p_{i,d}&=\frac{1}{2}\sign(\eta_{i,d} + a_{i,d})(|\eta_{i,d} + a_{i,d}|+L_d - u_{i,d} - |q_{i,d}| - \gamma_{i,d} )^+ & i\in[n],d\in[D] \\
        q_{i,d} &=\frac{1}{2}\sign(\zeta_{i,d} + b_{i,d})(|\zeta_{i,d} + b_{i,d}|+L_d - u_{i,d} - |p_{i,d}| - \gamma_{i,d} )^+  & i\in[n],d\in[D] \\
    u_{i,d} & = (-\gamma_{i,d}- |p_{i,d}| - |q_{i,d}|  + L_d )^+  & i\in[n],d\in[D]\\
        s_{i,j} &= (-\alpha_{i,j}- \hat y_i +  \hat y_j - z_i + z_j + \langle \bm a_i, \bm x_i-\bm x_j\rangle )^+  & i,j \in [n]  \\
     t_{i,j} &=  (-\beta_{i,j}- z_i + z_j + \langle \bm b_i, \bm x_i-\bm x_j\rangle )^+ & i,j \in [n]  \\
    \alpha_{i,j} &= \alpha_{i,j} + s_{i,j} + \hat y_i - \hat y_j + z_i - z_j - \langle \bm a_i, \bm x_i-\bm x_j\rangle & i,j \in [n] \\
    \beta_{i,j} &= \beta_{i,j} + t_{i,j} + z_i -  z_j - \langle \bm b_i, \bm x_i-\bm x_j\rangle & i,j \in [n] \\
    \gamma_{i,d} &= \gamma_{i,d} + u_{i,d} + |p_{i,d}| + |q_{i,d}|  - L_d  & i\in[n],d\in[D] \\
    \eta_{i,d} &=  \eta_{i,d} + a_{i,d}  -  p_{i,d} & i\in[n],d\in[D] \\
    \zeta_{i,d} &=  \zeta_{i,d} + b_{i,d}  -  q_{i,d} & i\in[n],d\in[D]
\end{align*}
    \ENDFOR
    \STATE {\bfseries Output:} 
       $\quad \hat{f}(\bm x) \triangleq  \max_{i\in [n]} \langle \bm a_i,  \bm x-\bm x_i \rangle + \hat y_i + z_i  -\max_{i\in[n]} \langle \bm b_i, \bm x-\bm x_i \rangle + z_i $
\end{algorithmic}
\end{algorithm*}

\subsection{Algorithm Derivation}

For the derivation of the ADMM algorithm, we start with a modified version of the empirical risk minimization in (\ref{program:SRM}); where we use a components wise regularization:
\begin{align*}
    \| f\| \triangleq \inf_{\phi_1,\phi_2}&\sum_{d=1}^D\sup_{\bm x} \sup_{ v_{i,d} \in \frac{\partial_*}{\partial {x_d}} \phi_i(\bm x)} | v_{1,d}| + |v_{2,d}| \\
    & \textrm{s.t. } \phi_1, \phi_2 \text{ are convex, } \phi_1 - \phi_2 = f,
\end{align*}
Hence the new empirical risk minimization becomes:
    \begin{align*}
        &\qquad \min_{\hat{y}_i, z_i, \bm a_i, \bm b_i, L} \sum_{i = 1}^n ( \hat{y_i} - y_i)^2 + \lambda \sum_{d=1}^D L_d\\
        &\textrm{s.t.} \begin{cases}
    \hat y_i - \hat y_j + z_i  - z_j\geq \langle \bm a_j, \bm x_i-\bm x_j\rangle & i,j \in [n] ,d\in[D]  \\
    z_i - z_j  \geq \langle \bm b_j, \bm x_i-\bm x_j \rangle    &  i,j \in [n] \\
    |a_{i,d}| + |b_{i,d}| \leq L_d & \quad i\in[n] ,d\in[D]
    \end{cases}\notag
    \end{align*}
First we change the inequality constraint into equality constraints. We further introduce some auxiliary variables to ease the solving of subsequent ADMM blocks. 
    \begin{align}\label{program:SRM_Admm_1}
        &\qquad \min_{\substack{\hat{y}_i, z_i, \bm a_i, \bm b_i, L\\ \bm p_i, \bm q_i, s_{i,j}, t_{i,j}, u_{i,d}}} \sum_{i = 1}^n  (\hat y_i - y_i)^2 + \lambda \sum_{d=1}^D L_d\\
        &\textrm{s.t.} \begin{cases}
     s_{i,j} + \hat y_i -  \hat y_j + z_i - z_j - \langle \bm a_i, \bm x_i-\bm x_j\rangle  = 0 & i,j \in [n]  \\
     t_{i,j} + z_i - z_j - \langle \bm b_i, \bm x_i-\bm x_j\rangle  = 0 & i,j \in [n]  \\
     |p_{i,d}| + |q_{i,d}| + u_{i,d}  - L_d = 0 & i\in[n],d\in[D] \\
    a_{i,d}  =  p_{i,d} & i\in[n],d\in[D]\\
    b_{i,d}  =  q_{i,d} & i\in[n],d\in[D]\\
    s_{i,j}, t_{i,j}, u_{i,d}  \geq 0 & i,j \in [n],d\in[D] \\
    \end{cases}\notag
    \end{align}
Then we write the augmented Lagrangian of ADMM as bellow. 
    \begin{align}\label{program:SRM_Admm_2}
        &\qquad \min \ell( \hat y_i, z_i, \bm a_i, \bm b_i ,\bm p_i, \bm q_i, L, s_{i,j}, t_{i,j},  \bm u_i,  \alpha_{i,j}, \beta_{i,j}, \bm \gamma_i, \bm \eta_i, \bm \zeta_i)\notag\\
        & =\quad \sum_{i = 1}^n  (\hat y_i - y_i)^2 +\lambda L \notag\\
        & + \quad \sum_{i}\sum_{j} \alpha_{i,j}\Big(s_{i,j} + \hat y_i - \hat y_j + z_i - z_j - \langle \bm a_i, \bm x_i-\bm x_j\rangle\Big) + \frac{\rho}{2}\Big(s_{i,j} + \hat y_i - \hat y_j + z_i - z_j - \langle \bm a_i, \bm x_i-\bm x_j\rangle\Big)^2 \notag\\
        &+ \quad \sum_{i}\sum_{j} \beta_{i,j}\Big(t_{i,j} + z_i -  z_j - \langle \bm b_i, \bm x_i-\bm x_j\rangle\Big) + \frac{\rho}{2}\Big(t_{i,j} + z_i -  z_j - \langle \bm b_i, \bm x_i-\bm x_j\rangle\Big)^2 \notag\\
         &+ \quad \sum_i\sum_d\gamma_{i,d}\Big (u_{i,d} - L_d + |p_{i,d}| + |q_{i,d}| \Big) + \frac{\rho}{2} \Big(u_{i,d} - L_d +  |p_{i,d}| + |q_{i,d}| \Big)^2 \notag\\
        &+ \quad\sum_{i}\sum_{d} \eta_{i,d}\Big(a_{i,d}  -  p_{i,d}\Big) + \frac{\rho}{2}\Big(a_{i,d}  -  p_{i,d}\Big)^2\notag\\
        &+ \quad \sum_{i}\sum_{d} \zeta_{i,d}\Big(b_{i,d}  -  q_{i,d}\Big) + \frac{\rho}{2}\Big(b_{i,d}  -  q_{i,d}\Big)^2\notag
    \end{align}
  Taking gradients with respect to the  augmented Lagrangian and solving for the roots provides the parallel algorithm in Appx~\ref{appx:admm_algor}. 

\newpage
\section{Datasets}
\begin{table}[h!]
\caption{Datasets used for the regression task. The entries in the first columns are linked to repositry copies of the same. The final column indicates if the DC function based method outperforms \emph{all} competing methods or not.}
\label{sample-table-reg}
\vskip 0.15in
\begin{center}
\begin{small}
\begin{sc}
\begin{tabular}{lcccr}
\toprule
Data set & Labels & Features & Did DC do Better? \\
\midrule

\href{http://www.stat.cmu.edu/~larry/all-of-nonpar/data.html}{Rock Samples}          & 16    & 3     & \cmark\\
\href{https://www.mathworks.com/help/stats/sample-data-sets.html}{Acetylene}          & 16    & 3     & \cmark\\
\href{https://www.mathworks.com/help/stats/sample-data-sets.html}{Moore}      & 20   & 5     & \cmark \\
\href{https://www.mathworks.com/help/stats/sample-data-sets.html}{Reaction}         & 13   & 3     & \cmark \\
\href{https://www.mathworks.com/help/stats/sample-data-sets.html}{Car small}                       &100   & 6    & \xmark \\
\href{https://www.mathworks.com/help/stats/sample-data-sets.html}{Cereal}                     & 77   & 12     &   \cmark      \\
\href{http://lib.stat.cmu.edu/datasets/boston}{Boston Housing}     & 506   & 13    & \xmark \\
\href{https://archive.ics.uci.edu/ml/datasets/Forest+Fires}{Forest Fires}     & 517   & 12    & \cmark \\
\bottomrule
\end{tabular}
\end{sc}
\end{small}
\end{center}
\vskip -0.1in
\end{table}
  
    \begin{table}[h]
\caption{Datasets used for the multi-class classification task. The entries in the first columns are linked to the UCI machine learning repositry copies of the same. The final column indicates if the DC function based method outperforms \emph{all} competing methods or not.}
\label{sample-table-class}
\vskip 0.15in
\begin{center}
\begin{small}
\begin{sc}
\begin{tabular}{lcccr}
\toprule
Data set & Labels & Features & Did DC do Better? \\
\midrule

\href{http://www.stat.cmu.edu/~larry/all-of-nonpar/data.html}{BPD}               & 223   & 1 & \cmark \\
\href{https://archive.ics.uci.edu/ml/datasets/Caesarian+Section+Classification+Dataset}{Iris}               & 150   & 4 & \cmark \\
\href{http://archive.ics.uci.edu/ml/datasets/balance+scale}{Balance Scale}      & 625   & 4     & \cmark \\
\href{https://archive.ics.uci.edu/ml/datasets/ecoli}{Ecoli}                     & 337   & 7     & \cmark       \\
\href{https://archive.ics.uci.edu/ml/datasets/Wine}{Wine}                       & 178   & 13    & \xmark \\
\href{https://archive.ics.uci.edu/ml/datasets/Heart+Disease}{Heart Disease}     & 313   & 13    & \xmark \\
\href{https://archive.ics.uci.edu/ml/datasets/ionosphere}{Ionosphere}           & 351   & 34    & \cmark \\
\bottomrule
\end{tabular}
\end{sc}
\end{small}
\end{center}
\vskip -0.1in
\end{table}

\end{appendix}


\end{document}